\documentclass[accepted]{uai2026} 
                        
\usepackage[american]{babel}

\usepackage{natbib} 
\usepackage{mathtools} 
\usepackage{booktabs} 
\usepackage{tikz} 

\usepackage{amsthm}
\usepackage{amsmath}
\usepackage{amsfonts}
\usepackage{mathtools}
\usepackage{algorithm}
\usepackage{algorithmic}
\usepackage{subcaption}

\theoremstyle{plain}
\newtheorem{theorem}{Theorem}[]
\newtheorem{proposition}[theorem]{Proposition}
\newtheorem{lemma}[theorem]{Lemma}

\theoremstyle{definition}

\theoremstyle{remark}

\theoremstyle{example}
\newtheorem{example}[theorem]{Example}

\title{DRL-ORA: Distributional Reinforcement Learning with Online Epistemic Risk Adaptation}

\author[1]{Yupeng Wu}{}
\author[2]{Wenyun Li}
\author[2]{Wenjie Huang}
\author[3]{Chin Pang Ho}
\affil[1]{%
London Business School
}
\affil[2]{%
The University of Hong Kong
}
\affil[3]{%
City University of Hong Kong
  }

\begin{document}
\maketitle

\begin{abstract}
One of the main challenges in reinforcement learning (RL) is that the agent has to make decisions that would influence the future performance without having complete knowledge of the environment. Dynamically adjusting the level of epistemic risk during the learning process can help to achieve reliable policies in safety-critical settings with better efficiency. In this work, we propose a new framework, \emph{Distributional RL with Online Epistemic Risk Adaptation} (DRL-ORA). This framework quantifies both \emph{epistemic} and implicit \emph{aleatory} uncertainties in a unified manner and dynamically adjusts the epistemic risk levels by solving a total variation minimization problem \emph{online}. The framework generalizes the existing variants of risk adaptation approaches with better explainability and flexibility.
The selection of risk levels is performed efficiently via 
a Follow-The-Leader-type algorithm.
We show that DRL-ORA outperforms existing methods that rely on fixed risk levels or manually designed risk level adaptation in multiple classes of tasks.

\end{abstract}

\section{Introduction}\label{sec:intro}
Reinforcement learning (RL) algorithms have shown great success in games and simulated environments \citep{mnih2015human}, and have drawn interest for real-world and industrial applications \citep{mahmood2018benchmarking,zhang2020cautious}. Taking a step further, distributional reinforcement learning (DRL) \citep{bellemare2017distributional} considers the intrinsic distribution of the returns rather than their expectation, as considered in standard RL settings. This approach allows agents to quantify the risk of their policies and enables the development of risk-aware policies by considering the distributional information of their performance \citep{dabney2018distributional}. Quantifying both aleatory and epistemic uncertainties of RL crucial for developing risk-aware policies. The former expresses the inherent \emph{randomness} in the problem, and the latter represents \emph{a lack of knowledge} about the environment \citep{mavrin2019distributional,clements2019estimating,charpentier2022disentangling}. However, most existing research has focused on risk-aware RL with a {\it fixed risk level}, such as investigating risk awareness in the return distribution \citep{lim2022distributional} and addressing both aleatoric and epistemic uncertainties \citep{eriksson2019epistemic,eriksson2022sentinel}.

In this paper, we investigate {\it dynamic risk-awareness} strategy with respect to epistemic uncertainty, 
which reflects the attitude toward exploring unknown environments. Modeling risk-awareness of epistemic uncertainty is particularly relevant for applications such as autonomous driving, where the systems are almost deterministic, and the primary source of uncertainty is the lack of information about the environment. A higher risk-aversion level can avoid the negative consequences of excessive exploration that may lead to unsafe decisions in real-life (pessimism-under-uncertainty). On the other hand, a lower risk-aversion level (including risk-seeking) encourages exploring and collecting more data in unexplored regions (optimism-under-uncertainty). Although standard DRL can learn a policy under different levels of risk-awareness and adapt to various environments that require different degrees of caution, a specific risk-awareness level typically needs to be pre-specified and fixed for each environment prior to deployment. Studies have shown that optimism and pessimism-under-uncertainty settings outperform each other dependent on the task at hand, highlighting the need for an adaptive risk-awareness strategy that can balance these seemingly opposing perspectives \citep{moskovitz2021tactical,wang2024improving,wang2025adaptive}. 

In addition, maintaining a fixed level of risk-awareness throughout the learning period is often practically suboptimal, as the level of risk concerned varies at different stages. An intelligent agent should be able to identify the appropriate level of risk required for effective decision-making at different stages, rather than simply following a fixed risk-aware policy \citep{schubert2021automatic,liu2022adaptive}. In the context of RL, problems where uncertainties are expected to diminish over time (e.g., epistemic uncertainty) should be assigned a decreasing risk-awareness level to avoid excessive conservativeness \citep{eriksson2019epistemic}. We illustrate this point further using a CartPole example. 
\begin{example}\label{ex:CartPole}
The CartPole task was trained through Implicit Quantile Network (IQN) with left-tail Conditional Value-at-Risk (CVaR) at level $\alpha\in (0,\,1]$ as its risk measure on the return (IQN will be explicitly described in Section 3). Over the episodes, the risk parameter $\alpha$ could either be fixed or manually adjusted. 
Figure~\ref{fig:CartPole_intro} shows that ``IQN alpha:191'' outperforms all other settings, including IQN with fixed risk levels. 
The intuition is: at early stage of training, set a high risk-awareness level ($\alpha=0.1$) to hedge the uncertain from the unknown environment, then move to a low risk-awareness level ($\alpha = 0.9$) to acquire high rewards as more experience from the environment is gained (and the epistemic uncertainty is reduced), and finally set back to a value in $[0,\,1]$ that acts optimally in hindsight. 
This result highlights the advantage of using varying risk levels and motivates the development of more advanced methods that can automatically adjust the risk level and achieve even better performance.
\end{example}
\begin{figure}[t]
\centering
    \includegraphics[width=0.45\textwidth]{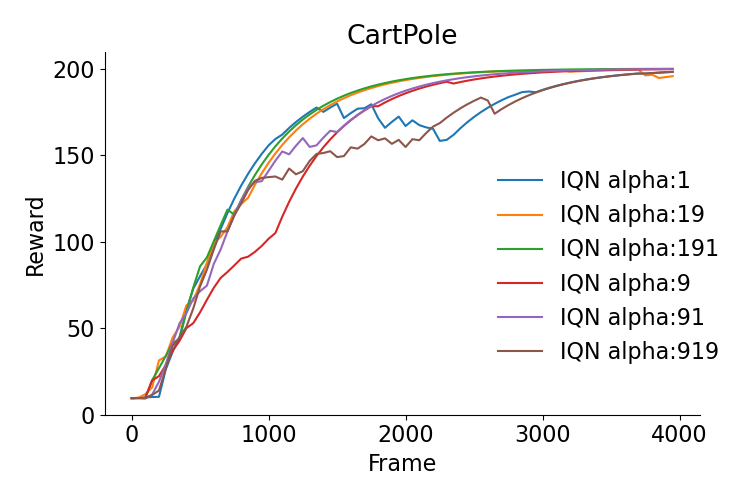}
    \caption{``IQN alpha:1'' represents that $\alpha$ is fixed at $0.1$ throughout all episodes. ``IQN alpha:191'' means that $\alpha$ is manually adjusted over the episodes, linearly increasing from $0.1$ to $0.9$ and then linearly decreasing back to $0.1$. Same interpretation applies to the other settings.}
    \label{fig:CartPole_intro}
\end{figure}

In this paper, we propose a DRL framework that is capable of automatically adjusting the risk-awareness level \emph{on the fly} (i.e., online) by the epistemic uncertainty quantification in each period. The proposed framework does not rely on any pre-specified risk-awareness level and it can be incorporated with a wide range of risk measures based on IQN. Our framework generalizes and improves the existing adaptive risk-awareness level approach \citep{moskovitz2021tactical,liu2022adaptive,wang2024improving,wang2025adaptive}. 
Dynamic risk level selection methods are beneficial for RL algorithms because predetermining a suitable risk-awareness level 
is difficult, if not impossible, for a new domain task without any prior knowledge.
Our experiments demonstrate that the proposed approach, which requires only moderate extensions to standard RL algorithms, achieves superior performance across various classes of practical problems compared to risk-neutral methods, fixed risk-awareness level methods, and existing adaptive risk-awareness level methods.





\section{Related Work}

The literature on DRL has significantly expanded in recent years \citep{morimura2010parametric, bellemare2017distributional, barth2018distributed, dabney2018implicit, yang2019fully, singh2020improving}. This approach has been used to enhance exploration in RL \citep{mavrin2019distributional} and in risk-aware applications \citep{bernhard2019addressing}. Many different types of risk measures have been used in risk-aware DRL including CVaR \citep{lim2022distributional, choi2021risk, keramati2020being, yang2022safety}, power-law risk measure \citep{choi2021risk}, distortion risk measure \citep{dabney2018implicit}, entropic risk measure \citep{hai2022rasr} and cumulative prospect theory \citep{prashanth2016cumulative}. Research has also been conducted on quantifying the risk of epistemic uncertainty, which demonstrates the attitude towards unexplored areas \citep{clements2019estimating,eriksson2019epistemic,tang2018exploration,o2023efficient}, and on establishing comprehensive composite risk, which quantifies the joint effect of aleatory and epistemic risk during the learning process \citep{eriksson2022sentinel,cho2023pitfall}.

The above literature focuses on \emph{fixed} risk-awareness levels. Recently, there has been growing evidence supporting the \emph{adaptive} risk-awareness algorithms, such as risk scheduling approaches that explore risk levels and optimistic behaviors from a risk perspective \citep{oh2022risk}. These methods are not situation-dependent, however, risk levels decay linearly from high to specific low levels during the learning process. 
In addition, frameworks for value estimation, such as Tactical Optimism and Pessimism (TOP) and Quantile Option Architecture (QUOTA) \citep{moskovitz2021tactical,wang2024improving,wang2025adaptive,zhang2019quota}, have been developed to balance the degree of optimistic and pessimistic uncertainty quantification. Many of these frameworks formulate the problem as a multi-armed bandit problem and utilize the Exponentially Weighted Average Forecasting (EWAF) algorithm.
To implement this method, a pre-specified degree level should be set based on expert recommendations. 
A similar EWAF algorithm is used for adaptive risk tendency (ART) 
of CVaR in an IQN framework for a Nano Drone navigation application \citep{liu2022adaptive}, and it achieves superior reward accumulation and efficient computational performance. The drawback of this method appears during evaluation in environments with high obstacle density, possibly because it relies solely on the current observation without utilizing any memory. Besides, EWAF is not intuitive 
, and tuning the learning rate for weight forecasting still requires extra effort. In another work \citep{schubert2021automatic}, risk levels are adjusted based on a random network distillation (RND) error, which provides an estimate of how frequently a state has been visited during training. While the performance heavily depends on the original RND architecture, it is possible that a more conservative risk setting could hinder exploration.
\section{Preliminaries}
\paragraph{Distributional Reinforcement Learning (DRL)} Consider a Markov Decision Process (MDP) \citep{puterman2014markov} defined as a tuple $(\mathbb{S},\mathbb{A},R,P,\gamma)$, where $\mathbb{S}$ is the state space, $\mathbb{A}$ the action space, $R$ is the (state and action-dependent) reward function, $P(\cdot|s,\,a)$ is the transition probability associated with state-action pair $(s,\,a)$, and $\gamma\in(0,1]$ is the discounted factor. A policy $\pi(\cdot|s)$ maps the state $s$ to a probability distribution over actions.

DRL \citep{bellemare2017distributional} introduces a new perspective other than the standard RL setting, by learning the entire return distribution of the action-value function $Q(s,\,a)$ directly instead of estimating only the expected return $Q(s,\,a)$. This approach allows the agent to capture more detailed information beyond a single expectation value. Given a policy $\pi$, the distributional Bellman equation is defined as
\[Z(s,\,a):\stackrel{D}{=}R(s,\,a)+\gamma Z\left(s^{\prime},\,a^{\prime}\right),\]
for $s^{\prime} \sim P(\cdot \mid s,\,a),\, a^{\prime}\sim \pi(\cdot|s^\prime)$, where $Z$ is the distributional value function and $\stackrel{D}{=}$ represents the equality in distribution.
Its corresponding optimality equation is defined as 
\[Z(s,\,a):\stackrel{D}{=} R(s,\,a)+\gamma Z\left(s^{\prime},\,\arg \max_{a^{\prime}\in\mathbb{A}} \mathbb{E} \left[Z\left(s^{\prime},\,a^{\prime}\right)\right]\right),\]
for $s^{\prime} \sim P(\cdot \mid s,\,a)$.
It is proved by \citet{bellemare2017distributional} that the distributional Bellman operator is a contraction in a maximal form of the Wasserstein
metric. There are approaches \citep{bellemare2017distributional,dabney2018distributional,dabney2018implicit,yang2019fully,rowland2019statistics} to estimate the return distribution, most of which are based on Quantile Regression \citep{koenker1978regression}. 

The Implicit Quantile Network (IQN) proposed by \citet{dabney2018implicit} provides an effective way to learn an implicit representation of the return distribution, yielding a powerful function
approximator for a Deep $Q$-Network agent. IQN can re-parameterize samples from a based distribution to the respective quantile value of a target distribution, which produces risk-aware policies. Let $\beta: [0,\,1] \to [0,\, 1]$ be a \emph{distortion} function, and choosing $\beta$ as the identity function 
corresponds to risk-neutrality. Let $Z_{q}(s,\,a)$ denote the $q$-quantile function (inverse of cumulative density function) of the return distribution, then the distorted expectation of $Z(s,\,a)$ under $\beta$ and its risk-aware greedy policy are denoted as
\begin{align}
Q_{\beta}(s,\,a)&:=\mathbb{E}_{q\sim U[0,1]}[Z_{\beta(q)}(s,\,a)], \label{distorted} \\
\pi_{\beta}(s)& =  \arg \max_{a\in\mathbb{A}}Q_{\beta}(s,\,a). \nonumber
\end{align}
For brevity, the IQN loss function and its detailed training process are summarized in Appendix Section A.

\paragraph{Adaptive Risk-awareness (Risk Tendency)} The proposed adaptive risk-awareness approach treats the selection of an exploration strategy (risk awareness level) as a multi-armed bandit problem \citep{moskovitz2021tactical, wang2024improving, wang2025adaptive}, where each bandit arm corresponds to a specific choice of distortion function $\beta$. Define a discrete set of $D$ distortion functions $\mathcal{D}=\{\beta^1,...,\beta^D\}$ recommended by experts. After sampling $d_n \in \{1,...,D\}$ at epsiode $n$, we form a belief distribution $p_n \in \Delta_{\mathcal{D}}$ ($\Delta_{\mathcal{D}}$ is the probability simplex on $\mathcal{D}$) of the form $p_n(d) \propto \exp(w_n(d))$ for all $d\in \mathcal{D}$. The distribution is updated using the Exponentially Weighted Average Forecasting (EWAF) algorithm based on a received feedback signal in each episode. After receiving a feedback signal $g_n: = R_n-R_{n-1}$ where $R_n$ is the cumulative reward obtained in episode $n$, the weight $w_n$ is updated by:
\[
w_{n+1}(d) = 
\begin{cases}
w_{n}(d) + \eta \cdot (g_n/p_n(d)), & d= d_n,
\\
w_{n}(d), & \text{otherwise},
\end{cases}
\]
where $\eta >0$ is the step-size. In each episode $n$, the $\beta$ in Eq.(\ref{distorted}) is chosen by random sampling from distribution $p_n$.

Another variants of constructing the feedback signal is the right-truncated variance difference of reward distributions between two successive periods (time-steps) \citep{liu2022adaptive}, where $g_t := \text{RTV}_t - \text{RTV}_{t-1}$ and 
\[
\text{RTV}_t: = \frac{2}{N} \sum_{i=1}^{N/2} (F^{-1}_{Z_t} (q_i) - F^{-1}_{Z_t} (q_{N/2}))^2,
\]
in which $q_i$ is the $i/N$th quantile levels and $N$ is a sample size. The EWAF algorithm is applied to update the weight $w_t$ in periodic frequency.

\section{Methodology: DRL-ORA}
{
We propose DRL-ORA, the first distributional reinforcement learning framework capable of online, state-action-level adaptation of epistemic risk attitudes without requiring pre-specified risk levels or manual scheduling. At its core, DRL-ORA introduces a unified uncertainty quantification scheme that explicitly disentangles epistemic from aleatory uncertainty via ensemble networks, and formulates risk adaptation as a non-convex online learning problem with a well-defined regret objective. This marks a fundamental departure from prior approaches—such as EWAF-based bandit selection—which lack explainability, rely on predefined discrete risk sets, and fail to leverage full distributional information. In contrast, DRL-ORA enables fine-grained, per-transition risk adaptation, grounded in a total variation minimization criterion that is both theoretically tractable and computationally efficient. The framework is agnostic to the choice of distortion risk measures, supports both CVaR and quantile-based variants, and establishes a novel connection to satisficing measures in decision theory—revealing that the offline oracle in our formulation corresponds to a quasi-concave satisficing objective. Together, these innovations yield a flexible, interpretable, and high-performing risk adaptation paradigm that consistently outperforms fixed-risk and existing adaptive baselines across diverse tasks.
}
\paragraph{A Generalized Non-convex Learning Perspective} 
In {{the existing}} EWAF approach, the major shortcoming is its explainability. The theoretical foundation/objective of the approach in updating the {{risk tendency}} is unclear. The most similar framework to the approach is online optimization on simplex \citep{bubeck2011introduction}, where the signals must be an explicit loss function in the decision variable, rather than simple numerical outputs. There are also several other shortcomings. In this approach, epistemic uncertainty is not disentangled from the reward distribution, which means that the agent cannot specify the risk-awareness level towards exploration as it will further lead to inaccurate adaptation. Besides, it relies on truncated variance or cumulative reward, rather than leveraging the full information from the epistemic uncertainty distribution, to construct the feedback signal. Furthermore, this approach requires a predefined discrete finite set of distortion functions (e.g., only two functions are used in \citep{liu2022adaptive}). 

Instead, we propose a non-convex learning perspective that constructs the loss function as the feedback signal, generalizing the variants of the existing approaches, and offering greater explainability, flexibility and improved accuracy. Given $T$ as the total periods (time-steps) of training, we then define a loss function $l_t: \mathcal{A} \rightarrow \mathbb{R}$ for all $t = 1, \ldots, T$. Our feedback signal at period $t$ is:
\begin{align}
 & l_t(\alpha(s,\,a)) \nonumber
 \\
 := & \big| \rho_{\alpha(s,\,a)}(X_t(s,\,a)) - \rho_{\alpha(s,\,a)}(X_{t+1}(s,\,a)) \big|, \label{signal}   
\end{align}
where $X_t(s,\,a)$ represents the epistemic uncertainty associated with the state-action pair $(s,\,a)$ at period $t$. Here we define the parametric risk measure on the epistemic uncertainty: $\rho_{\alpha}: \mathcal{L} \rightarrow \mathbb{R}$, where $\mathcal{L}$ represents the space of bounded random variables, $\alpha \in \mathcal{A}$ is a parameter that controls the level of risk-awareness, and $\mathcal{A} \subset \mathbb{R}$ is a compact set containing all possible risk parameters. Explicitly, $\alpha(s,\,a) \in \mathcal{A}$ denotes the risk parameter chosen corresponding to the state-action pair $(s,\,a)$. 

The risk measure $\rho_{\alpha}(\cdot)$ is Lipschitz continuous but not necessarily convex with respect to $\alpha$. 
Distortion risk measures \citep{wang1996premium,dabney2018distributional} fall under our framework and are well-suited for the IQN framework. Specifically, for $X \in \mathcal{L}$, we have
\[\rho_{\alpha}(X) := \int_{-\infty}^{0}[g_{\alpha}(S_{X}(x))-1]dx + \int_{0}^{\infty} g_{\alpha}(S_{X}(x))dx,\] 
where $g_{\alpha}$ is the distortion function parameterized by $\alpha$, and $S_{X}(x)$ is the survival distribution function of random variable $X$. Suppose $g_{\alpha}$ is a bounded and Lipschitz continuous function with respect to $\alpha$. Typically, when $X$ represents a random \emph{loss}, a larger risk parameter $\alpha$ indicates a lower level of risk-awareness. Below, we provide several examples along with their corresponding $g_{\alpha}$ and $\mathcal{A}$.

\begin{example}[\citep{acerbi2002coherence}]
For the right-tail CVaR, $\text{CVaR}_{\alpha}(X) = \mathbb{E}[X|X \geq F^{-1}(1-\alpha)]$, where $F^{-1}$ denotes the inverse CDF of $X$. The corresponding distortion function is $g_{\alpha}(x) = \max\left\{-x/\alpha + 1,\,0\right\}$, and $\mathcal{A} = [\alpha_{\min},\,1]$ for some $\alpha_{\min} \in (0, 1]$.
\end{example}
\begin{example}[\citep{dhaene2012remarks}]
For the right-tail $\alpha$-quantile, \emph{i.e.}, $\rho_{\alpha}(X) = F^{-1}_{X}(1-\alpha)$, the distortion function is $g_{\alpha}(x) = \mathbb{I}\{x \leq \alpha\}$, and $\mathcal{A} = [\alpha_{\min},\,1]$ for some $\alpha_{\min} \in (0, 1]$.
\end{example}

Our construction of the feedback signal (\ref{signal}) has several features and advantages:
\begin{enumerate}
    \item[\emph{(i)}]  Our feedback signal generalizes the existing approach by utilizing the full information of the epistemic uncertainty distribution. Also, the feedback signal is defined as an explicit function of the risk parameter, so the risk adaptation problem can be formally written as an online learning problem with clear objective and decision foundation.

    \item[\emph{(ii)}] Our approach provides higher flexibility, as the risk levels associated with different state-action pairs can be updated independently and simultaneously. Furthermore, the adaptation is performed in every period/transition rather than episodically, as in \citep{moskovitz2021tactical, wang2024improving, wang2025adaptive}. These will be validated later by experiments, demonstrating improved training performance during early periods and episodes.
    
    \item[\emph{(iii)}] The risk measure $\rho_\alpha(\cdot)$ focuses on the right $\alpha$-level quantile of a distribution, capturing the risk of left-truncated epistemic uncertainty (reflecting upper-tail variability), which makes it comparable with the setting of reward distribution RTV in \citep{liu2022adaptive}. It is inspired by setting a contrary to \citep{mavrin2019distributional}, which argues that the upper-tail variability of reward distribution reflects optimism of exploration in the face of uncertainty, particularly when the distribution is asymmetric.
\end{enumerate}
The offline objective, in hindsight, is formulated for each state-action pair $(s,\,a)$ as
\begin{equation}
    \min_{\alpha(s,\,a) \in \mathcal{A}} \sum_{t=1}^{T-1} l_t(\alpha(s,\,a)), \label{Obj}
\end{equation}
which can be interpreted as finding a state-action-dependent risk parameter that minimizes the \emph{Total Variation} \citep{chambolle2004algorithm, huang2008fast, li2013efficient} (i.e., the sum of feedback signals) of the epistemic uncertainty risk. 

The entire learning period under hindsight reveals the absolute level of risk awareness associated with a particular risk parameter that stabilizes the negative impact of epistemic uncertainty variation. However, the relative levels of risk awareness during the process cannot be explicitly elicited. At any period $t < T$, only limited information $(X_{1}, \ldots, X_{t+1})$ is available, allowing the best risk parameter to be forecasted on the fly using an appropriate online learning method. The ``regret'' of the online learning algorithm under the outputs $\alpha_1, \ldots, \alpha_T$ is defined as  
\[
\mathfrak{R}_T := \sum_{t=1}^{T} l_t(\alpha_t) - \min_{\alpha \in \mathcal{A}} \sum_{t=1}^{T} l_t(\alpha),
\]
where the dependence of $\alpha$ and $X_t$ on $(s,\,a)$ is omitted only for simplicity, as will be done in the following discussions.

\paragraph{Ensemble Networks for Epistemic Uncertainty Quantification}
Building on the principles of Bayesian inference \citep{mackay2003information}, the \emph{ensemble network} is a straightforward approach for modeling epistemic uncertainty and is commonly used in prior work \citep{dimitrakakis2006nearly, osband2016deep, agarwal2020optimistic}. Recall that in the regular RL framework, a neural network $\theta$ is used to estimate the $Q$-function $Q^\theta(s, a)$, where the distribution of $\theta$ reflects the current epistemic uncertainty during training \citep{eriksson2022sentinel}. While modeling the distribution of $\theta$ is challenging, its corresponding network outputs $Q^\theta(s,\,a)$ can be directly observed.

Thus, ensemble networks create $K$ network heads during training, each initialized with different parameters $\theta$, denoted by $\{\theta_i\}_{i=1}^K$. For each state-action pair $(s,\,a)$, these $K$ networks output $K$ different values $Q^{\theta_k}(s,\,a)$ for $k = 1, \ldots, K$. The distribution of $Q^{\theta_k}(s,\,a)$ is used to approximate the corresponding distribution of $\theta$, which represents the epistemic uncertainty at $(s,\,a)$. Formally, for each $(s,\,a)$, the ensemble networks generate $Q^{\theta_k}(s,\,a)$, and a distribution $Y$ maps $(s,\,a)$ to a uniform probability distribution supported on $\{Q^{\theta_k}(s,\,a)\}$, defined as: \[Y(s,\,a) = 1/K \sum_{k=1}^{K} \delta\{Q^{\theta_k}(s,\,a)\},\] where $\delta\{z\}$ is the Dirac delta function at $z \in \mathbb{R}$. Thus, the epistemic uncertainty associated with $(s,\,a)$ is represented by the distribution $Y(s,\,a)$. At time step $t-1$ in the RL process, given the explored state-action pair $(s_{t-1},\,a_{t-1})$ and the epistemic uncertainty distribution $X_{t-1}(s,\,a)$ for each $(s,\,a)$, the epistemic uncertainty distribution $X_{t}(s,\,a)$ is updated asynchronously as
\[
\begin{array}{ll}
X_{t}(s,\,a)\leftarrow Y(s_{t-1},\,a_{t-1}), & \textrm{if}\,(s,\,a)=(s_{t-1},\,a_{t-1}),\\
X_{t}(s,\,a)\leftarrow X_{t-1}(s,\,a), & \textrm{otherwise},
\end{array}
\]
where $Y(s_{t-1},\,a_{t-1})$ is the distribution generated by the ensemble networks. This distribution is supported by the $Q$-functions of each network, where each $Q$-value is computed as the quantile expectation over the return distribution (aleatoric uncertainty). 
It is important to note that the return distribution in each network is updated based on the output of its corresponding ensemble network. This ensures that the epistemic uncertainty is properly quantified and disentangled from the overall uncertainties. {{Appendix Figure~\ref{fig:ablation k} shows the tests on different ensemble size $K$. It demonstrates that the increase of ensemble network size does not exhibit a significant correlation with performance improvement. Both the time and memory overhead grow slightly faster than linear growth with
$K$. With the default value of $K$, this overhead is not significant.}}

We finally wrap-up the proposed algorithm -- \emph{Distributional RL with Online {{Epistemic}} Risk Adaptation (DRL-ORA)} in Algorithm 1. In each period $t$, the epistemic uncertainty is estimated using ensemble networks, where the return distribution is updated based on the output of each ensemble network (Line 9). The historical epistemic uncertainties are then used to dynamically update the risk parameter through the proposed {{online non-convex learning method (Line 10-11)}}. 
The composite risk framework quantifies the combined effects of aleatory and epistemic risks by applying separate risk-awareness levels to the epistemic uncertainty and the return distribution (which implicitly captures aleatory uncertainty) using distortion functions $g_{\alpha}$ and $\beta$, respectively (Line 5 \& 7). The action is chosen to minimize the epistemic uncertainty risk under the latest estimation of the parameter (Line 5).

\begin{algorithm}[tb]
\caption{DRL-ORA}
\begin{algorithmic}[1]
\STATE \textbf{Input:} Aleatory uncertainty distortion function $\beta$; epistemic uncertainty distortion function $g_{\alpha}$ with risk parameter $\alpha\in[\alpha_{\min},\,\alpha_{\max}] \subset [0,\,1]$ (the risk parameter is associated with each state-action pair) and corresponding risk measure $\rho_\alpha$; {{parameter of exponential distribution $\eta$, arbitrarily chosen $\alpha_{1}\in\mathcal{A}^\prime$}}, hyper-parameters: $M>0$, $K>0$, discounted factor $\gamma>0$; Training period $T>0$.
\STATE \textbf{Initialize:} The return distribution $Z^k$ for each ensemble network $k=1,...,K$, risk parameter $\alpha(s,\,a)=\alpha_{\max}$ for all $(s,\,a)$, initial state $s_{0}$ and initial action $a_{0}$, initial estimated epistemic uncertainty distribution $X_{1}(s,\,a)$ for all $(s,\,a)$.
\medskip
\FOR{$ t= 1$ {\bfseries to} $T$}
\STATE Set $\bar{\alpha}_{t}(s) =  \alpha_{t}(s,\,a_{t-1})$.
\STATE Choose an action by 
\[a_{t} \in \arg\min_{a\in\mathcal{A}}\rho_{\bar{\alpha}_{t}({s_t)}}(X_{t}(s_t,a)).\] Observe return $r_t$ and state transition $s_{t+1}$. 
\FOR{$ k= 1$ {\bfseries to} $K$}
\STATE Compute $Q^{\theta_{k}}(s,a) = \frac{1}{M}\sum_{m=1}^{M}Z^k_{\beta(\tilde{q}_{m})}(s,\,a)$ by $M$ samples of $\tilde{q}_{m}\sim U([0,\,1])$, where the return distribution $Z^k$ is updated based on the $k$th ensemble network's output.
\ENDFOR
\STATE Update the estimated epistemic uncertainty distribution $X_{t+1}(s,\,a)$ for all $(s,\,a)$ asynchronously.
\STATE Generate randomly $\sigma_{t+1}\overset{\mathrm{i.i.d}}{\sim} \exp{(\eta)}$;
\STATE Compute \[\alpha_{t+1} \in \arg \min_{\alpha \in\mathcal{A}^\prime } \left\{\sum_{i=1}^{t} l_i(\alpha) - \sigma_{t+1}\alpha \right\}.\]
\ENDFOR
\end{algorithmic}
\end{algorithm}
\paragraph{Regret Minimization and Analysis}
Given that our loss function is not necessarily convex in the risk parameter $\alpha$, standard online convex optimization algorithms cannot guarantee sublinear regret. The EWAF algorithm also fails in this setting, as our goal is to identify a specific risk parameter through continuous optimization rather than updating logit weights over a predefined discrete set. The Follow the Perturbed Leader (FTPL) algorithm \citep{agarwal2019learning} achieves sublinear regret complexity $O(T^{1/2})$ for the expected regret $\mathbb{E}[\mathfrak{R}_T]$ when the perturbation $\{\sigma_t\}_{t \geq 0}$ is sampled from an exponential distribution with parameter $\eta$. Theorem 4 demonstrates that a proper discretization of $\mathcal{A}$ into a finite set $\mathcal{A}' \subseteq \mathcal{A}$, combined with grid search over $\mathcal{A}'$, can efficiently solve the offline oracle problem while preserving the $O(T^{1/2})$ expected regret complexity. This method is particularly suitable when gradient or subgradient information is difficult to obtain.
\begin{theorem}
For an arbitrarily small $\epsilon>0$,  the set $\mathcal{A}$ can be properly discretized as $\mathcal{A}^{\prime}$, such that 
the Hausdorff distance between the two sets, i.e., \[\max\left\{\sup_{\alpha\in\mathcal{A}}\inf_{\alpha^{\prime} \in \mathcal{A}^{\prime}} |\alpha - \alpha^{\prime}|,\,\inf_{\alpha\in\mathcal{A}}\sup_{\alpha^{\prime} \in \mathcal{A}^{\prime}} |\alpha - \alpha^{\prime}|\right\}\leq \epsilon.\] 
In addition, by choosing $\epsilon = O(T^{-1/2})$, {{Algorithm 1}} can achieve $O(T^{1/2})$ expected regret complexity.
\end{theorem}

\paragraph{Extension: Relation to Satisficing Measure} Now we define a \emph{recursive} loss function $l^{\prime}_t: \mathcal{A} \times \mathcal{A} \rightarrow \mathbb{R}$ for all $t = 1, \ldots, T$, where the first argument represents the decision variable $\alpha$, and the second argument is the parameter corresponding to the solution from the previous period. Specifically, at period $t-1$, we define \[l^{\prime}_{t-1}(\alpha,\,\alpha_{t-1}) := |\rho_{\alpha}(X_{t-1}) - \rho_{\alpha_{t-1}}(X_{t})|.\]
It can be shown that the cumulative regret under the above recursive loss function is equivalent to the cumulative regret under the regular loss function; that is, \[\mathfrak{R}^{\prime}_{T}:=\sum_{t=1}^{T} l^{\prime}_{t}(\alpha_{t},\,\alpha_{t}) -  \min_{\alpha\in \mathcal{A}} \sum_{t=1}^{T} l^{\prime}_{t}(\alpha,\,\alpha) = \mathfrak{R}_{T}.\]

We then modify and implement {{Algorithm 1  with the offline oracle model (Line 11)}} as 
\[
\alpha_{t} \in \arg \min_{\alpha \in\mathcal{A}^\prime} l^{\prime}_{t-1}(\alpha,\,\alpha_{t-1}),
\]
where the random perturbation term is omitted. Using this recursive loss function requires significantly less storage space compared to the regular loss, but the tradeoff is that the modified {{Algorithm 1}} incurs an $O(T)$ expected regret complexity for both $\mathfrak{R}_T$ and $\mathfrak{R}'_T$, which implies that, at least, a local optimum of the hindsight problem is attained.

The connection between this algorithm's offline oracle model and the {\it satisficing measure} \citep{brown2009satisficing} from the decision analysis literature, can be established. The satisficing measure evaluates the extent to which a random outcome fails to meet a fixed target. To formalize this connection. Let $\tau_t := \rho_{\alpha_{t-1}}(X_t)$ denote the ``target'' at period $t$. Given $\Omega$ as the scenario space, if $\min_{\omega \in \Omega} X_{t-1}(\omega) \geq \tau_t$, then $\alpha_t = \alpha_{\min}$ (which could be $0$, depending on the chosen risk measure). Conversely, if $\max_{\omega \in \Omega} X_{t-1}(\omega) \leq \tau_t$, then $\alpha_t = \alpha_{\max}$ (which could be $1$, depending on the chosen risk measure). Otherwise, the offline oracle problem of {{Algorithm 1}} becomes equivalent to solving the following quasi-concave optimization problem:
\begin{equation}
    \max \left\{ 1 - \alpha \, : \, \rho_{\alpha}(X_{t-1}) \leq \tau_t, \, \alpha \in \mathcal{A} \right\}, \label{satisficing}
\end{equation}
where the optimal solution provides the level of non-attainability of $X_{t-1}$ relative to $\tau_t$. 
Theorem \ref{thm:satisficing}, 
based on the analysis in Section 3.4 of \citet{brown2009satisficing} and Theorem 3.6 of \citet{huang2020data}, demonstrates that when $\rho_{\alpha}$ is chosen as CVaR, Problem \eqref{satisficing} can be reformulated as a stochastic convex program.

\begin{theorem}\label{thm:satisficing}
Suppose $\rho_{\alpha} = \text{CVaR}_{\alpha}$. We have \[\alpha_{t} = \min_{b\geq0}\mathbb{E}\left[\left(b(X_{t-1}-\tau_{t})+1\right)_{+}\right],\] with $\tau_{t} = \rho_{\alpha_{t-1}}(X_{t})$.
\end{theorem}
After applying sample average approximation for the expectation, Problem (\ref{satisficing}) can be further reformulated as a linear programming (LP) problem. The number of decision variables and constraints are both $O(K)$, and the best-known LP algorithm has a complexity of $O(K^{2.5})$ \citep{vaidya1989speeding}. However, we develop a specialized search algorithm, illustrated in Appendix Section B, which can determine $\alpha_t$ with significantly lower computational complexity.
\begin{theorem}
The search algorithm (Appendix Section B: Algorithm~\ref{Search Algorithm}) computes the optimal risk parameter in time $O(K \log K)$.
\end{theorem}

\section{Experiments and Applications}

We conduct experiments on three classes of tasks to evaluate the performance of DRL-ORA. In all experiments, CVaR is used in ORA-based methods with $\alpha_{\min} = 0.1$, resulting in $\mathcal{A} = [0.1, 1]$. Additionally, we set $\beta(q) = q$, making the agent risk-neutral with respect to the implicit aleatory uncertainties of the tasks. We also implement the two existing risk aware adaptation approaches: ART \citep{liu2022adaptive} and TOP \citep{moskovitz2021tactical,wang2024improving,wang2025adaptive} as benchmarks.

\paragraph{Atari Games} We first demonstrate the performance of DRL-ORA on CartPole and the Atari games: Hero, MsPacman and SpaceInvaders. Here we analyze the results of CartPole task in addition to Example \ref{ex:CartPole}. Figure~\ref{fig:CartPole} in Appendix, built upon Figure~\ref{fig:CartPole_intro}, shows that DRL-ORA outperforms all other methods, with a noticeable reward advantage at the early stages of the episodes. 
{To statistically validate the observed performance differences, we employed the Mann-Whitney $U$ test \citep{10.1214/aoms/1177730491}. The resulting Rank-Biserial correlation effect sizes (0.990 against ART and 0.787 against TOP, $p$ < 0.001) confirm a substantial advantage of ORA over both the ART and TOP baselines in the CartPole environment.}
Furthermore, the 90\% confidence interval of DRL-ORA highlights its stability. 
{To evaluate the robustness of our approach, we compared the ORA using quantile ("ORA(quantile)") against the default CVaR-based version ("ORA(CVaR)"). The results show that despite a slight performance reduction, ORA(quantile) maintains considerable risk adaptation capability.  The subsequent experiments further demonstrate the robustness of ORA across different risk measures.}
Due to limited spaces, please refer to Appendix Section C for the additional experiment results on other Atari games. 

\paragraph{Nano Drone Navigation}
We evaluate DRL-ORA on an open-source Nano Drone navigation task \citep{liu2022adaptive}, where ART has shown superior performance compared to IQN. In this partially observable environment, the drone aims to maximize its score by reaching the target while avoiding obstacles. It gains points for success but loses points for time and collisions, and the agent's level of risk-awareness directly influences its exploration strategy and overall performance. We record the complete training and testing results for ORA with the regular loss function (``Recursive ORA'' in short), ORA with the recursive loss function, ART, TOP and IQN under various fixed risk parameters. All training parameters, including learning rate and random seeds, are consistent with the default settings in \citep{liu2022adaptive}, see Appendix Table~\ref{Supp:Hyper-parameters N}.

\begin{figure}[t]
\begin{center}
\includegraphics[width=0.45\textwidth]{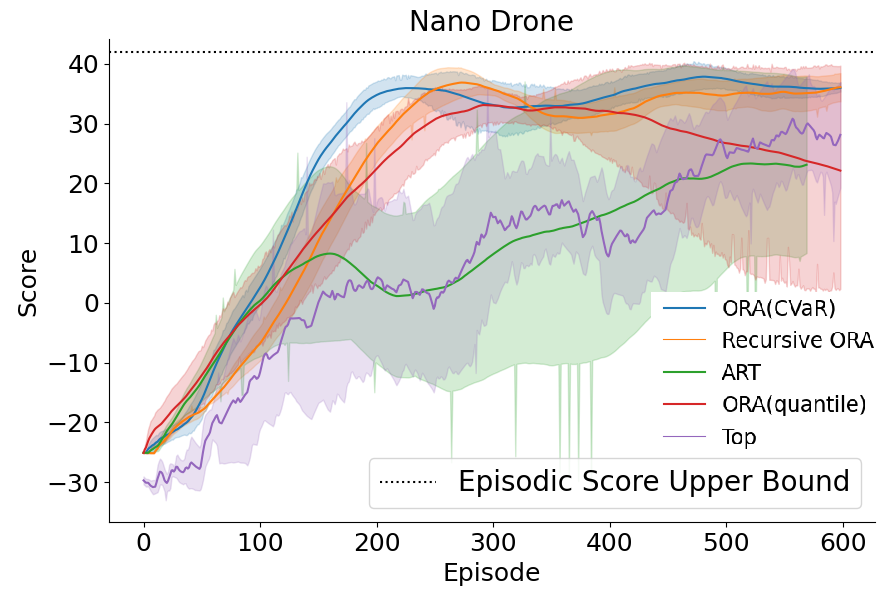}
    \end{center}
    \caption{Average episodic scores in Nano Drone navigation task. The shaded area represents a 90\% confidence interval.} 
    \label{fig:NanoDrone reward}
\end{figure}

\begin{table}[h]
  \centering
      \caption{Success and collision rate percentage}
      \label{tab:Testing}
      \smallskip
      \small{
\begin{tabular}{lcccccc}
\toprule
 & \multicolumn{2}{c}{Density 2} & \multicolumn{2}{c}{Density 6} & \multicolumn{2}{c}{Density 12} \\
\cmidrule(lr){2-3} \cmidrule(lr){4-5} \cmidrule(lr){6-7}
Metric (\%) & Suc. & Col. & Suc. & Col. & Suc. & Col. \\
\midrule
IQN $\alpha$=0.5 & 88 & 5 & 70 & 15 & 48 & 32 \\
IQN $\alpha$=1.0 & 88 & 4 & 74 & 10 & 52 & 25 \\
ART & 84 & 9 & 73 & 15 & 54 & 25 \\
Recursive ORA & {93} & 5 & 75 & 16 & 54 & 26 \\
TOP & 90 & 1 & 78 & 5 & 60 & 26 \\
\textbf{ORA(CVaR)} & 92 & 6 & \textbf{81} & 13 & \textbf{64} & 28 \\
\textbf{ORA(quantile)}&\textbf{94}&6&75&16&52&23\\
\bottomrule
\end{tabular}
}

\end{table}
As shown in Figure~\ref{fig:NanoDrone reward}, ORA not only shows the best training performance and fastest convergence 
{but also significantly outperforms both baselines ($p$ < 0.001), achieving Rank-Biserial effect sizes of 0.319 over ART and 0.309 over TOP.}
Table~\ref{tab:Testing} also presents the average testing results of each algorithm over 500 evaluation runs in environments with varying obstacle densities (i.e., the number of obstacles). Testing is particularly important in this navigation task, as excessive exploration during training may result in lower rewards but ultimately lead to an effective policy with strong testing performance. The testing results show that ORA outperforms all other algorithms, especially in environments with higher uncertainty (\emph{Density 6 \& 12}). In contrast, in the task with lower environmental uncertainty (\emph{Density 2}), Recursive ORA slightly outperforms ORA while offering more efficient computation and requiring less storage for historical epistemic uncertainties.

\paragraph{Knapsack}
We assess DRL-ORA by applying it to a classic operations research (OR) problem, the Knapsack problem, on OR-gym \citep{hubbs2020or, balaji2019orl}, a widely utilized RL library for OR problems. The Knapsack problem is a combinatorial optimization problem aiming at maximizing the total value of items placed in a knapsack without exceeding its maximum capacity. The agent's reward is the value of items placed in the knapsack, and an episode terminates once the capacity is full. There is no aleatory uncertainty in the Knapsack problem. We set \emph{50 items} in the Binary Knapsack \citep{hubbs2020or} for a lower computational load while keeping other settings aligned with the default. All training hyper-parameters are provided in Appendix Table~\ref{Supp:Hyper-parameters}. In terms of testing results in Figure~\ref{fig: Knapsack_testing}, ORA achieves a higher average reward than IQN, demonstrating that by precisely capturing epistemic uncertainty, our adaptive algorithm maintains its effectiveness in such scenarios.
{Our ORA method achieves perfect separation from TOP (effect size = 1.000) and substantial advantages over DQN (0.699) and ART (0.544), all with $p$ < 0.001.}
Note that the testing result should be normally higher than the final training reward because a fixed $\epsilon$ probability for random action selection is used in the whole training, while there is no random action selection in testing. 

\begin{figure}[t]
    \centering
\includegraphics[width=0.45\textwidth]{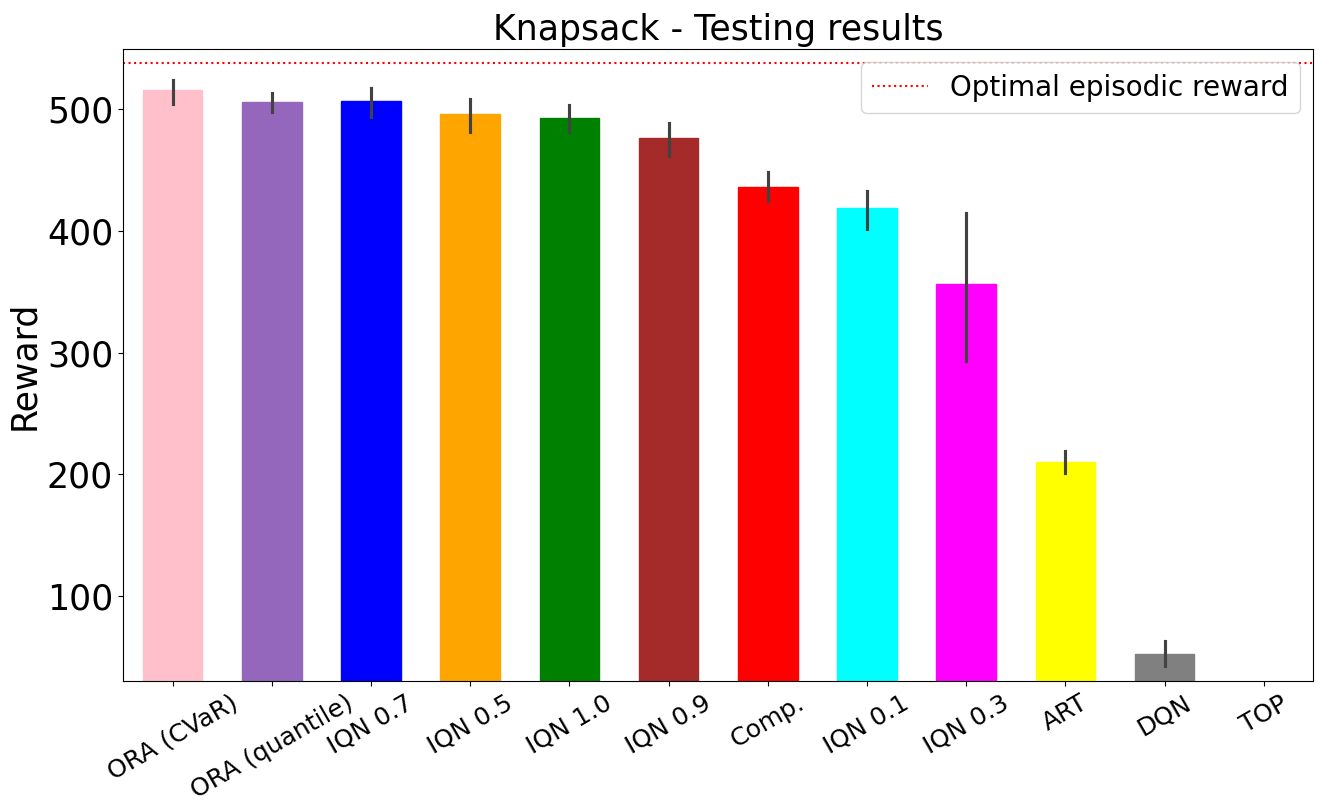}
    \caption{Testing results with 90\% confidence interval. "Comp." means Composite IQN. The ``Optimal episodic reward'' is the benchmark solved via DP.}
    \label{fig: Knapsack_testing} 
    \end{figure}
    \begin{figure}[t]
   \centering
\includegraphics[width=0.45\textwidth]{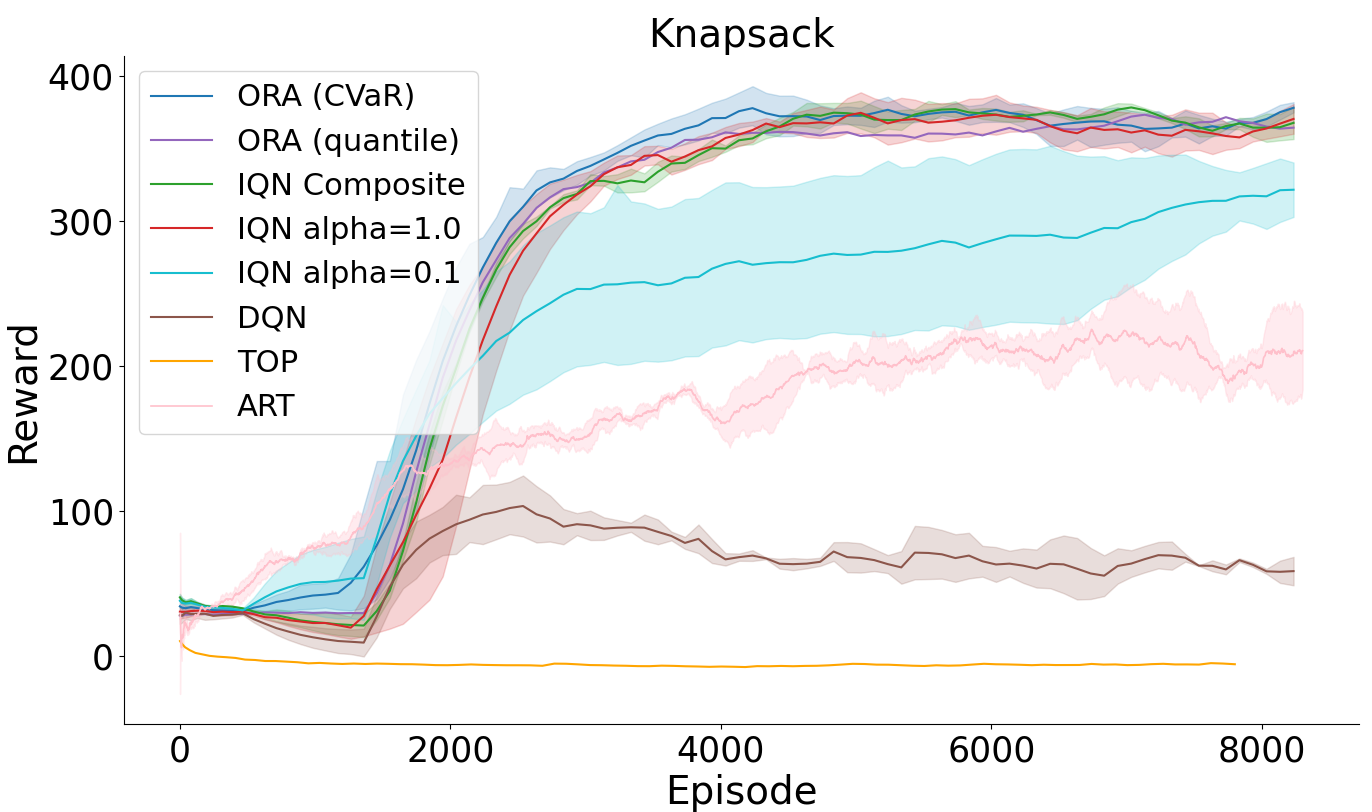}
    \caption{Reward lines on Knapsack. }
    \label{fig: Knapsack_reward}
\end{figure}

In Figure~\ref{fig: Knapsack_reward}, IQN with $\alpha = 1.0$ demonstrates a clear advantage over IQN with $\alpha = 0.1$, while regular Deep $Q$-Network (DQN) \citep{mnih2015human}. fails to perform effectively. This suggests that lower risk awareness is more appropriate in tasks without aleatory uncertainty. Once again, ORA outperforms the best IQN, TOP and ART, particularly during the early and intermediate stages of training. Additionally, we conduct an ablation experiment comparing ORA with the IQN Composite method \citep{hai2022rasr}. The IQN Composite algorithm is derived from Algorithm 1 by fixing $\alpha_t = 1$ for all $t = 1, \ldots, T$ and for all state-action pairs, making it risk-neutral with respect to epistemic uncertainty. This method replaces the expectation imposed on the return distribution with an expectation imposed solely on epistemic uncertainties, as there is no aleatory uncertainty in this task.
Figure~\ref{fig: Knapsack_reward} shows that IQN Composite fails to achieve similar performance to ORA in the early training episodes, indicating that ORA's superior performance stems from its adaptive mechanism for risk level selection.

\paragraph{Further Remarks} 
Table~\ref{Relation between alpha and epistemic uncertainty} in Appendix, reports, for ORA, the monitored left-truncated variance (LTV) with respect to the median of the epistemic uncertainty distribution obtained in each training period of the three experiments. The table also shows the corresponding risk parameters used in each case.
CartPole has the smallest LTV value in average. Its LTV fluctuates over the whole training process, and the corresponding risk parameter remains at a high level. We also observe that LTV is almost uni-modal over episodes in Nano Drone navigation and Knapsack. A high risk parameter (lower risk-awareness level), is favored at the beginning to acquire knowledge of the environment, whereas it subsequently declines in performance when large LTVs are detected. Knapsack has the largest LTV value which decreases slower than that in Nano Drone navigation after reaching the peak. This can explain why the superior performance of ORA lasts longer in the Knapsack problem than that in Nano Drone navigation (see Figure~\ref{fig:NanoDrone reward} and~\ref{fig: Knapsack_reward}). 


\section{Conclusion}

In this work, we propose an RL framework with an adaptive mechanism for learning risk levels to hedge against the epistemic uncertainties. The method only needs an extra multi-head ensemble structure so is eligible of employing different loss functions and distortion risk measures and being applied to IQN-type algorithms. Various applications validate the superiority of the proposed method. For future research, we plan to develop more accurate and efficient methods for epistemic uncertainty quantification, e.g., improving the scalability when a large number of ensemble networks are applied. We will also extend our approach to RL tasks under non-stationary environments. {{Our approach is expected to remain valid under non-stationary environment as it is adaptive according to the latest experience/episode (epistemic uncertainty estimation).}}





\bibliography{AAMAS/ICML}

\newpage

\onecolumn

\title{Appendix\\DRL-ORA: Distributional Reinforcement Learning with Online Epistemic Risk Adaptation}
\maketitle

\appendix
\section{Preliminaries on the IQN Loss Function}

For two samples $q,\,q^{\prime}$ and the reward $r_t$ at time step $t$, the TD error is $\delta_t^{q,\,q^{\prime}}=r_t+\gamma Z_{q^{\prime}}(s_{t+1}, \pi_\beta(s_{t+1}))-Z_q(s_t,\,a_t)$. The IQN loss function is given by:
\[
\mathcal{L}\left(s_t, a_t, r_t, s_{t+1}\right)=\frac{1}{N^{\prime}} \sum_{i=1}^N \sum_{j=1}^{N^{\prime}} h_{q}^\kappa\left(\delta_t^{q_i,\, q_j^{\prime}}\right),
\]
where $N$ and $N^\prime$ denote the respective number of i.i.d. samples $q_{i},\,q_{j^{\prime}} \sim U([0,\,1])$ used to estimate the loss, and $h_{\tau}^\kappa$ is the quantile estimate generated by the Huber loss function (see \citet{huber1992robust}) with a hyper-parameter threshold $\kappa$ to smooth the non-differentiable point at zero; that is, 
\begin{align*}
h_{\tau}^\kappa\left(\delta_t^{q_i,\,q_j^{\prime}}\right)&=\left|q-\mathbb{I}\left\{\delta_t^{q_i,\, q_j^{\prime}}<0\right\}\right| \frac{L_\kappa\left(\delta_t^{q_i,\, q_j^{\prime}}\right)}{\kappa};\\
L_\kappa\left(\delta_t^{q_i,\, q_j^{\prime}}\right)&=\begin{cases}\frac{1}{2} (\delta_t^{q_i,\,q_j^{\prime}})^2, & \text { if }\left|\delta_t^{q_i,\,q_j^{\prime}}\right| \leq \kappa, \\ \kappa\left(\left|\delta_t^{q_i,\, q_j^{\prime}}\right|-\frac{1}{2} \kappa\right), & \text { otherwise }.\end{cases}
\end{align*}
where $\delta_t$ is the TD-error, $\mathbb{I}\{\cdot\}$ is the indicator function, and normally $\kappa = 1$ is chosen (see \citet{dabney2018implicit}).
A corresponding sample-based risk-aware policy is then obtained by approximating $Q_{\beta}$ 
through the average of distorted return distribution samples. Based on the above settings, we are able to estimate the return distribution through a neural network structure, update the network parameters by by minimizing the IQN loss, and implement a risk-aware policy.

\section{Proofs for Technical Results}



\textsc{Proof of Theorem 4}: Based on Lemma 4.1 of \citet{cesa2006prediction} and \citet{agarwal2019learning}, we can instead analyze a slightly different algorithm which draws only a single noise
vector $\sigma \sim \exp(\eta)$, rather than drawing a fresh noise vector on every round. Proving regret bounds for this variant translates into asymptotically equivalent (expected) regret bounds for Algorithm 1. Set $l_0 = -\sigma \cdot \alpha$. We have the following several lemmas and proposition  hold which are used to prove Theorem 4.
\begin{lemma}
 The loss function $l_{t}$ is Lipschitz continuous in $\alpha$.
\end{lemma}
\begin{proof}
    Use the fact that the sum of Lipschitz functions is also Lipschitz, and the reverse triangle inequality for $L_{2}$-norm.
\end{proof}
Denote the Lipschitz continuity modulus as $G>0$ for $\{l_{t}\}_{t=0}^{T}$.
\begin{lemma} \label{A1}
    By the discretization set $\mathcal{A}^\prime \subseteq \mathcal{A}$, we have from Lipschitz continuity of loss functions that,
\begin{equation*}
    \min_{\alpha^{\prime} \in \mathcal{A}^{\prime}}\sum_{t=0}^{T} l_{t}(\alpha^{\prime}) - \min_{\alpha \in \mathcal{A}} \sum_{t=0}^{T} l_{t}(\alpha) \leq (T+1)G\epsilon. \label{Dis}
\end{equation*}
\end{lemma}
\begin{proof}
    Denote $\alpha^1 \in \arg\min_{\alpha \in \mathcal{A}} \sum_{t=0}^{T} l_{t}(\alpha)$ and $\alpha^2 \in \arg \min_{\alpha^{\prime} \in \mathcal{A}^{\prime}}\sum_{t=0}^{T} l_{t}(\alpha^{\prime})$. We have by Lipschitz continuity of loss function that 
    \[
     \left| \sum_{t=0}^{T} l_{t}(\alpha^1) -\sum_{t=0}^{T} l_{t}(\alpha^2)\right|\leq (T+1)G |\alpha^1 - \alpha^2|.
    \]
    And we further have 
     \[\left| \sum_{t=0}^{T} l_{t}(\alpha^1) -\sum_{t=0}^{T} l_{t}(\alpha^2)\right|\leq (T+1)G |\alpha^1 - \alpha^2| \leq  (T+1)G\epsilon,
     \]
     given that the Hausdorff distance is bounded by $\epsilon$. The proof can be completed given that 
     \[
     \min_{\alpha^{\prime} \in \mathcal{A}^{\prime}}\sum_{t=0}^{T} l_{t}(\alpha^{\prime}) - \min_{\alpha \in \mathcal{A}} \sum_{t=0}^{T} l_{t}(\alpha) \geq 0.
     \]
\end{proof}

\begin{lemma}\label{A2}
The regret is upper bounded via, 
\[
\sum_{t=0}^{T} l_{t}(\alpha_{t}^{\prime}) - \min_{\alpha^{\prime}\in\mathcal{A}^{\prime}}\sum_{t=0}^{T} l_{t}(\alpha^{\prime}) \leq \sum_{t=0}^{T} l_{t}(\alpha_{t}^{\prime}) - \sum_{t=0}^{T} l_{t}(\alpha_{t+1}^{\prime}).
\]
\end{lemma}
\begin{proof}
    The proof is inspired by the proof of Lemma 2.1 in \citet{shalev2012online}. Subtracting $\sum_t l_{t}(\alpha_t^{\prime})$ from both sides of the inequality and rearranging, the desired inequality can be rewritten as
    \[
    \sum_{t=0}^{T} l_{t}(\alpha_{t+1}^{\prime}) \leq \min_{\alpha^\prime \in\mathcal{A}^\prime }\sum_{t=0}^{T} l_{t}(\alpha^\prime).
    \]
    We prove this inequality by induction. The base case of $T = 0$ follows directly from the definition of $\alpha_{t+1}^{\prime}$. Assume the inequality holds for $T-1$, then for all $\alpha^\prime \in \mathcal{A}^\prime$, we have 
    \[
    \sum_{t=0}^{T-1} l_{t}(\alpha_{t+1}^{\prime}) \leq \sum_{t=0}^{T-1} l_{t}(\alpha^\prime).
    \]
    Adding $l_T(\alpha_{T+1}^{\prime})$ to both sides we get
    \[
      \sum_{t=0}^{T} l_{t}(\alpha_{t+1}^{\prime}) \leq l_T(\alpha_{T+1}^{\prime}) + \sum_{t=0}^{T-1} l_{t}(\alpha^\prime).
    \]
    The above holds for all $\alpha^\prime$ and in particular for $\alpha^\prime  = \alpha^\prime_{T+1}$. Thus,
    \[
      \sum_{t=0}^{T} l_{t}(\alpha_{t+1}^{\prime}) \leq  \sum_{t=0}^{T} l_{t}(\alpha_{T+1}^\prime) = \min_{\alpha^\prime \in\mathcal{A}^\prime }\sum_{t=0}^{T} l_{t}(\alpha^\prime),
    \]
    where the last equation follows from the definition of $\alpha_{T+1}^\prime$. This concludes our inductive argument.
\end{proof}
By combining Lemma \ref{A1} and \ref{A2}, we have 
\[
\sum_{t=0}^{T} l_{t}(\alpha_{t+1}^{\prime}) \leq \min_{\alpha\in\mathcal{A}}\sum_{t=0}^{T} l_{t}(\alpha) + (T+1)G\epsilon.
\]
We further have 
\begin{equation}
\sum_{t=0}^{T} l_{t}(\alpha^{\prime}_{t}) -  \min_{\alpha \in \mathcal{A}} \sum_{t=0}^{T} l_{t}(\alpha) \leq \sum_{t=0}^{T} l_{t}(\alpha^{\prime}_{t}) - \sum_{t=0}^{T} l_{t}(\alpha^{\prime}_{t+1}) + (T+1)G\epsilon.\label{Base}
\end{equation}

\begin{lemma}
The expected regret of Algorithm 1, from inequality (\ref{Base}) has
\begin{align}
& \mathbb{E}\left[\sum_{t=1}^{T} l_{t}(\alpha^{\prime}_{t}) -  \min_{\alpha \in \mathcal{A}} \sum_{t=1}^{T} l_{t}(\alpha)\right]  \nonumber
\\
\leq & \mathbb{E}\left[-\sigma \alpha ^{\ast} + \sigma \alpha^{\prime}_{1} \right] + \sum_{t=1}^{T} \mathbb{E}\left[l_{t}(\alpha^{\prime}_{t}) - l_{t}(\alpha^{\prime}_{t+1})\right]+(T+1)G\epsilon, \label{second}
\end{align}
where $\alpha^{\ast} \in \arg\min \{\sum_{t=1}^{T} l_{t}(\alpha),\,\alpha \in\mathcal{A}\}$.    
\end{lemma}
\begin{proof}
Using the result of Lemma \ref{A1} and \ref{A2}, we have 
    \[
    \sum_{t=0}^{T} l_{t}(\alpha_{t}^{\prime}) - \min_{\alpha\in\mathcal{A}}\sum_{t=0}^{T} l_{t}(\alpha) - (T+1)G\epsilon \leq \sum_{t=0}^{T} l_{t}(\alpha_{t}^{\prime}) - \min_{\alpha^{\prime}\in\mathcal{A}^{\prime}}\sum_{t=0}^{T} l_{t}(\alpha^{\prime}) \leq \sum_{t=0}^{T} l_{t}(\alpha_{t}^{\prime}) - \sum_{t=0}^{T} l_{t}(\alpha_{t+1}^{\prime}).
    \]
    By rearranging the above terms (taking the $l_0$ out), and taking expectation over both sides of the inequality, we can get the desired result.
\end{proof}
\begin{proposition}
    The first two terms at the right hand side of the inequality (\ref{second}) have a complexity bound $O(T^{1/2})$.
\end{proposition}
\begin{proof}
     The proof is inspired by the proof of Theorem 1 in \citet{agarwal2019learning}. For any two functions $f_1,\,f_2: \mathcal{A}^\prime \rightarrow \mathbb{R}$ and scalar $\sigma_1,\,\sigma_2 \in \mathbb{R}$, let 
    \[
    \alpha_i(\sigma_i) \in \arg \min_{\alpha \in \mathcal{A}^\prime} \left\{f_i(\alpha) - \sigma_i \alpha\right\},\,i=1,2.
    \]
    Let $f = f_1 - f_2$ and $\sigma = \sigma_1 - \sigma_2$, we have that 
    \[
    f_1(\alpha_1(\sigma_1)) - \sigma_1 \alpha_1(\sigma_1) \leq   f_1(\alpha_2(\sigma_2)) - \sigma_1 \alpha_2(\sigma_2), 
    \]
    and 
    \[
    f_2(\alpha_2(\sigma_2)) - \sigma_2 \alpha_2(\sigma_2) \leq   f_1(\alpha_1(\sigma_1)) - \sigma_2 \alpha_1(\sigma_1). 
    \]
    By adding the above two inequalities and rearranging, we have 
    \[
    f_1(\alpha_1(\sigma_1)) -  f_2(\alpha_2(\sigma_2)) \leq \sigma (\alpha_1(\sigma_1)- \alpha_2(\sigma_2)).
    \]
    For any round $t$, consider substituting $f_1(\alpha) = \sum_{i<t} l_i(\alpha),\,f_2(\alpha) = \sum_{i<t+1} l_i(\alpha),\,\sigma_2 = \sigma$ and $\sigma_1 = \sigma^\prime = \sigma + 2G$. We immediately get that 
    \[
    l_t(\alpha_t(\sigma^\prime)) - l_t(\alpha_{t+1}(\sigma^\prime)) \leq 2G (\alpha_t(\sigma^\prime) - \alpha_{t+1}(\sigma)).
    \]
    Using the fact that $l_t$ is $G$-Lipschitz, we get that 
    \[
    -G|\alpha_t(\sigma^\prime)-\alpha_{t+1}(\sigma)| \leq l_t(\alpha_t(\sigma^\prime)) - l_t(\alpha_{t+1}(\sigma^\prime)) \leq 2G (\alpha_t(\sigma^\prime) - \alpha_{t+1}(\sigma)),
    \]
    which immediately implies that $\alpha_t(\sigma^\prime) \geq \alpha_{t+1}(\sigma)$. Similar calculations show that $\alpha_{t+1}(\sigma^\prime) \geq \alpha_{t}(\sigma)$ and $\alpha_{t}(\sigma^\prime) \geq \alpha_{t}(\sigma)$. In the following proof, we omit the dependence on $t$. We denote by $\alpha_{\min}(\sigma) = \min\{\alpha_{t}(\sigma),\,\alpha_{t+1}(\sigma)\},\,\alpha_{\max}(\sigma) = \max\{\alpha_{t}(\sigma),\,\alpha_{t+1}(\sigma)\}$. First we observe that 
    \[
    \mathbb{E}[|\alpha_{t}(\sigma)-\alpha_{t+1}(\sigma)|] =  \mathbb{E}[\alpha_{\max}(\sigma)] - \mathbb{E}[\alpha_{\min}(\sigma)].
    \]
    Secondly the computation above implies that
    \begin{equation}
     \alpha_{\min}(\sigma^\prime) \geq \alpha_{\max}(\sigma). \label{minmax}
    \end{equation}
    Letting $\sigma^\prime = \sigma +2G$, we have 
    \begin{align*}
    \mathbb{E}[\alpha_{\min}(\sigma)] = & \int_{\sigma = 0}^{2G} \eta \exp(-\eta \sigma) \alpha_{\min}(\sigma)d\sigma + \int_{\sigma > 2G} \eta \exp{-\eta \sigma} \alpha_{\min}(\sigma) d \sigma 
    \\
    \geq & (1-\exp(-2\eta G))(\mathbb{E}[\alpha_{\max}(\sigma)]-1) + \int_{\sigma >0} \eta \exp(-\eta\sigma^\prime) \alpha_{\min}(\sigma^\prime)d\sigma
    \\
    \geq & (1-\exp(-2\eta G))(\mathbb{E}[\alpha_{\max}(\sigma)]-1) + \int_{\sigma >0} \eta \exp(-\eta\sigma^\prime) \alpha_{\max}(\sigma)d\sigma
    \\
    = & (1-\exp(-2\eta G))(\mathbb{E}[\alpha_{\max}(\sigma)]-1)  + \exp(-2\eta G)\mathbb{E}[\alpha_{\max}(\sigma)]
    \\
    = & \mathbb{E}[\alpha_{\max}(\sigma)] - (1-\exp(-2\eta G)) \geq \mathbb{E}[\alpha_{\max}(\sigma)] - 2\eta G,
    \end{align*}
    where the second inequality uses inequality (\ref{minmax}) and and the last inequality uses the inequality $\exp(x) \geq 1+x$.

    Finally we have that 
    \begin{align*}
     \mathbb{E}\left[-\sigma \alpha ^{\ast} + \sigma \alpha^{\prime}_{1} \right] + \sum_{t=1}^{T} \mathbb{E}\left[l_{t}(\alpha^{\prime}_{t}) - l_{t}(\alpha^{\prime}_{t+1})\right] 
     \leq & \mathbb{E}[\sigma] + G \sum_{t=1}^{T} \mathbb{E}[\|\alpha^{\prime}_{t} -\alpha^{\prime}_{t+1}\|_{1}]   
     \\
     \leq 1/\eta + 2\eta G^2 T.
    \end{align*}
And by choose $\eta = T^{-1/2}$, we get the desired result.
\end{proof}
Finally, if we choose $\epsilon = O(T^{-1/2})$ in inequality (\ref{second}), then the whole expected regret preserves the $O(T^{1/2})$ complexity. We thus complete the proof of Theorem 4.
\\
\\
\textsc{Proof of Theorem 5}: From Problem (4) in main paper, we have 
\begin{align*}
& \max\{1-\alpha:\,\textrm{CVaR}_{\alpha}(X_{t-1}-\tau_{t})\leq0\}
\\
=	&\max\left\{ 1-\alpha:\,\exists y\leq0:\,y+\alpha^{-1}\mbox{\ensuremath{\mathbb{E}}}\left[\left(X_{t-1}-\tau_{t}-y\right)_{+}\right]\leq0\right\} 
\\
=	&\max\left\{ 1-\alpha:\,\exists y\leq0:\,-1 + \alpha^{-1}\mbox{\ensuremath{\mathbb{E}}}\left[\left(-(X_{t-1}-\tau_{t})/y+1\right)_{+}\right]\leq0\right\} 
\\
=	&\max\left\{ 1-\alpha:\,\exists y\leq0:\,1-\alpha\leq1-\mbox{\ensuremath{\mathbb{E}}}\left[\left(-(X_{t-1}-\tau_{t})/y+1\right)_{+}\right]\right\}
\\
=	&\max_{b\geq0}\left\{ 1-\mbox{\ensuremath{\mathbb{E}}}\left[\left(b(X_{t-1}-\tau_{t})+1\right)_{+}\right]\right\} .
\end{align*}
Thus problem $\max\left\{ \alpha:\,\textrm{CVaR}_{\alpha}(X_{t-1})\leq\tau_{t}\right\}$ is equivalent to 
\begin{equation}
\min_{b\geq0}\mathbb{E}\left[\left(b(X_{t-1}-\tau_{t})+1\right)_{+}\right]. \label{minimization}
\end{equation}
\textsc{Proof of Theorem 6}: 
Problem (\ref{minimization}) can be approximated by a constrained linear program with $O(K)$ decision variables and $O(K)$ constraints, using sample average approximation. Define $\boldsymbol{c} \in \mathbb{R}^K$, and we have:
\begin{align*}
    \min_{\boldsymbol{c},b\geq0} ~~ & \frac{1}{K} \sum_{k=1}^{K} c_k
    \\
    \text{s.t.} ~~  & c_k \geq b(Q_{t-1}^{\theta_k}-\tau_{t})+1, &\forall k = 1,...,K,
    \\
    & c_k \geq 0, &\forall k = 1,...,K.
\end{align*}
where $Q_{t-1}^{\theta_k}$ is the output of the $k$-th ensemble network at period $t-1$ (we omit its dependence on state-action pair for simplicity). Using the structure of Problem (\ref{minimization}), we can further reduce the computation. The approximation of Problem (\ref{minimization}) is given by 
\begin{equation}
\min_{b\geq0}\frac{1}{K}\sum_{k=1}^{K}\left[b(Q_{t-1}^{\theta_k}-\tau_{t})+1\right]_{+}, \label{Transmin}
\end{equation}
which contains the sum of positive part functions where each positive part has slope $Q_{t-1}^{\theta_k}-\tau_{t}$ and each function pass the point $(0,\,1)$. Problem (\ref{Transmin}) can be graphically illustrated by a simple case in Figure~\ref{fig:transmin}, where there are exactly two functions, both with positive and negative slopes,  respectively. Given $b > 0$, suppose $S_1 + S_2 \geq S_3 + S_4 \geq 0$, then we must have $\alpha_t = 1$, otherwise, we will find the smallest $b$-intercept (i.e., $b_1$ and $b_2$), such that the sum of function gaps below 1 is lower than that above 1. Here we would check if at $b_2$, that $S_1 + S_2 > S_3$ holds. 
\begin{figure}[h]
    \centering
    \includegraphics[width=0.55\linewidth]{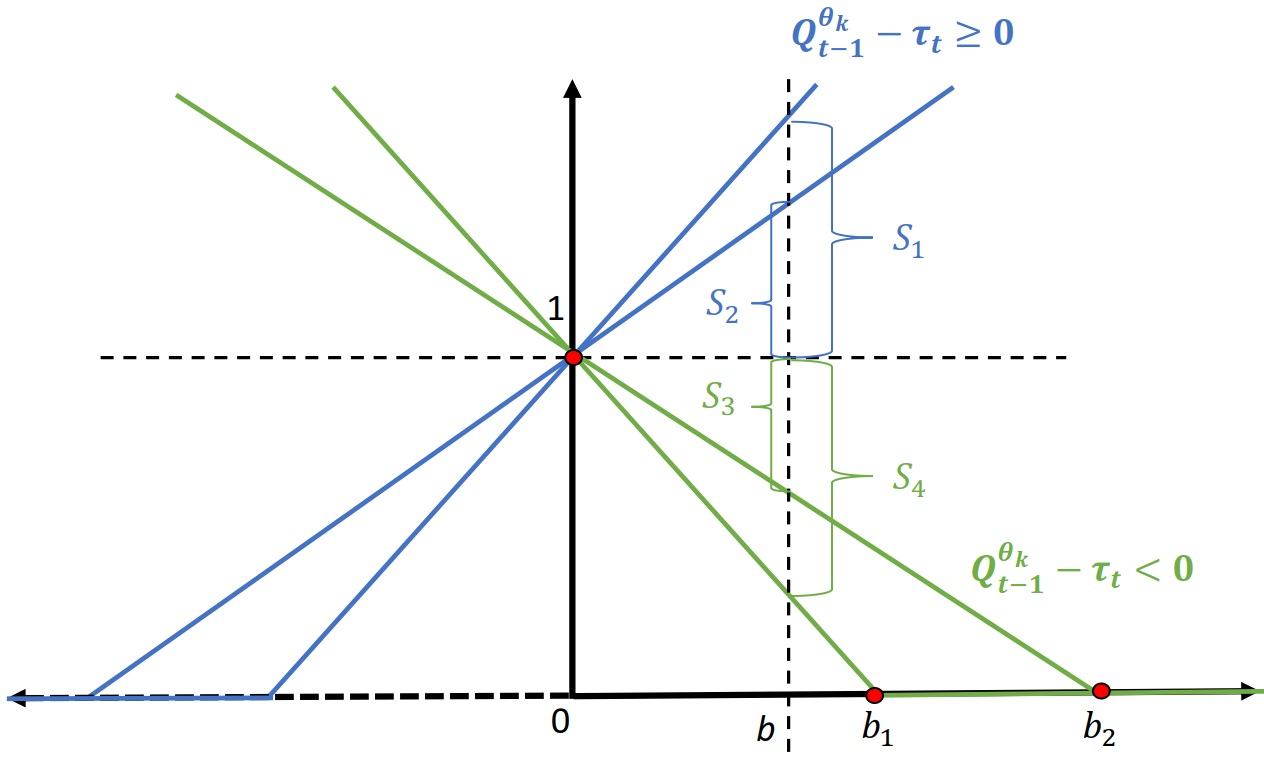}
    \caption{Graphic Illustration of Problem (\ref{Transmin})}
    \label{fig:transmin} 
\end{figure}

By generalizing the above idea into the general case, we can design a search algorithm: Algorithm 3 to find $\alpha_t$ efficiently. 

\begin{algorithm}[H]
\caption{Search Algorithm for $\alpha_t$}\label{Search Algorithm}
\begin{algorithmic}
\STATE Denote the index set $[K^\prime]$ where for each $k\in [K^\prime]$ where $Q^{\theta_k}_{t-1} - \tau_t < 0$. 

\IF{$\sum_{k\in [K]/[K^\prime]} (Q^{\theta_k}_{t-1} - \tau_t) \geq \sum_{k\in [K^\prime]} (\tau_t-Q^{\theta_k}_{t-1})$,}
\STATE Set $\alpha_t = 1$, 
\ELSE
\STATE Redefine and reorder the elements (the $b$-intercepts) in $[K^\prime]$ such that 
\[\frac{1}{\tau_t-Q^{\theta_1}_{t-1}} \leq \cdots \leq \frac{1}{\tau_t-Q^{\theta_{|[K^\prime]|}}_{t-1}}.\]
Then, $b^\ast = 1/(\tau_t-Q^{\theta_{k^\ast}}_{t-1})$, where $k^\ast \in\{1,...,|[K^\prime]|\}$ and \[
\sum_{i=1}^{k^{\ast}+1} (\tau_t-Q^{\theta_{k^\ast}}_{t-1}) \leq \sum_{k\in [K]/[K^\prime]} (Q^{\theta_k}_{t-1} - \tau_t)<\sum_{i=1}^{k^{\ast}} (\tau_t-Q^{\theta_{k^\ast}}_{t-1}).
\]
\ENDIF
\end{algorithmic}
\end{algorithm}

The most computational consuming step in Algorithm is the sorting $b$-intercepts in the set $[K^\prime]$ which has complexity $O(K\log(K))$, which much lower than that of the best known algorithm for linear program $O(K^{2.5})$ (see \citet{vaidya1989speeding}).

\section{Supplementary Materials for Experiments}
\begin{figure}
    \centering
    \includegraphics[width=\linewidth]{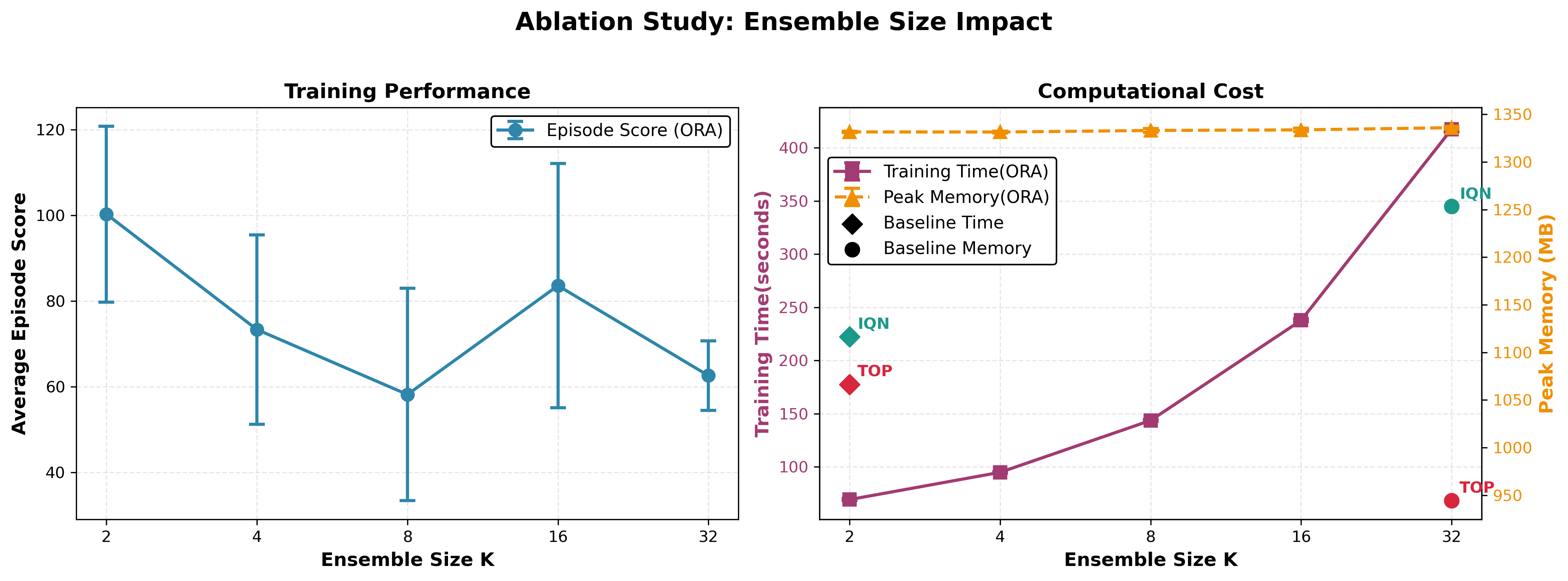}
    \caption{{Impact of ensemble size on ORA's training performance and computational cost in CartPole environment. Results averaged over 3 runs per configuration (N = 2, 4, 8, 16, 32). Error bars show standard deviation. The time and memory overheads of IQN and TOP are indicated on the left and right sides of the second subplot, respectively.}}
    \label{fig:ablation k}
\end{figure}

\begin{figure}[t]
\centering
    \includegraphics[width=0.5\textwidth]{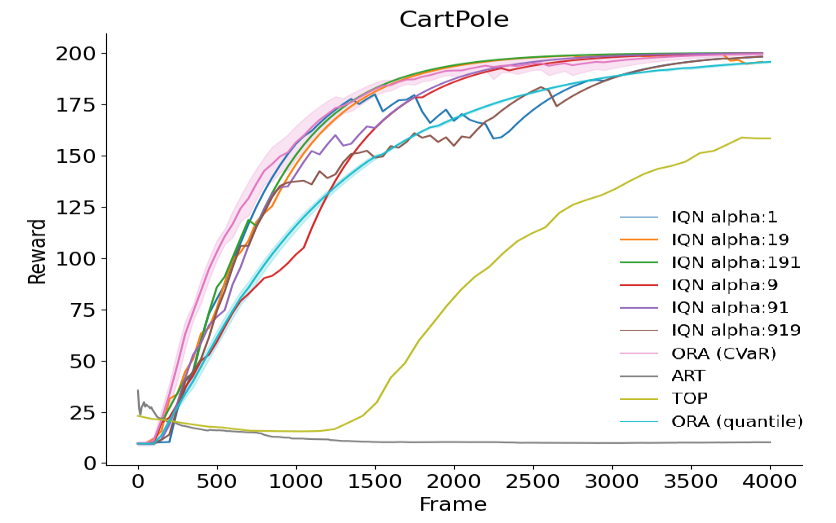}
    \caption{ORA with 90\% confidence interval.}
    \label{fig:CartPole}
\end{figure}

\begin{figure}[htbp]
    \centering
    \begin{subfigure}[b]{0.5\textwidth}
        \includegraphics[width=\textwidth]{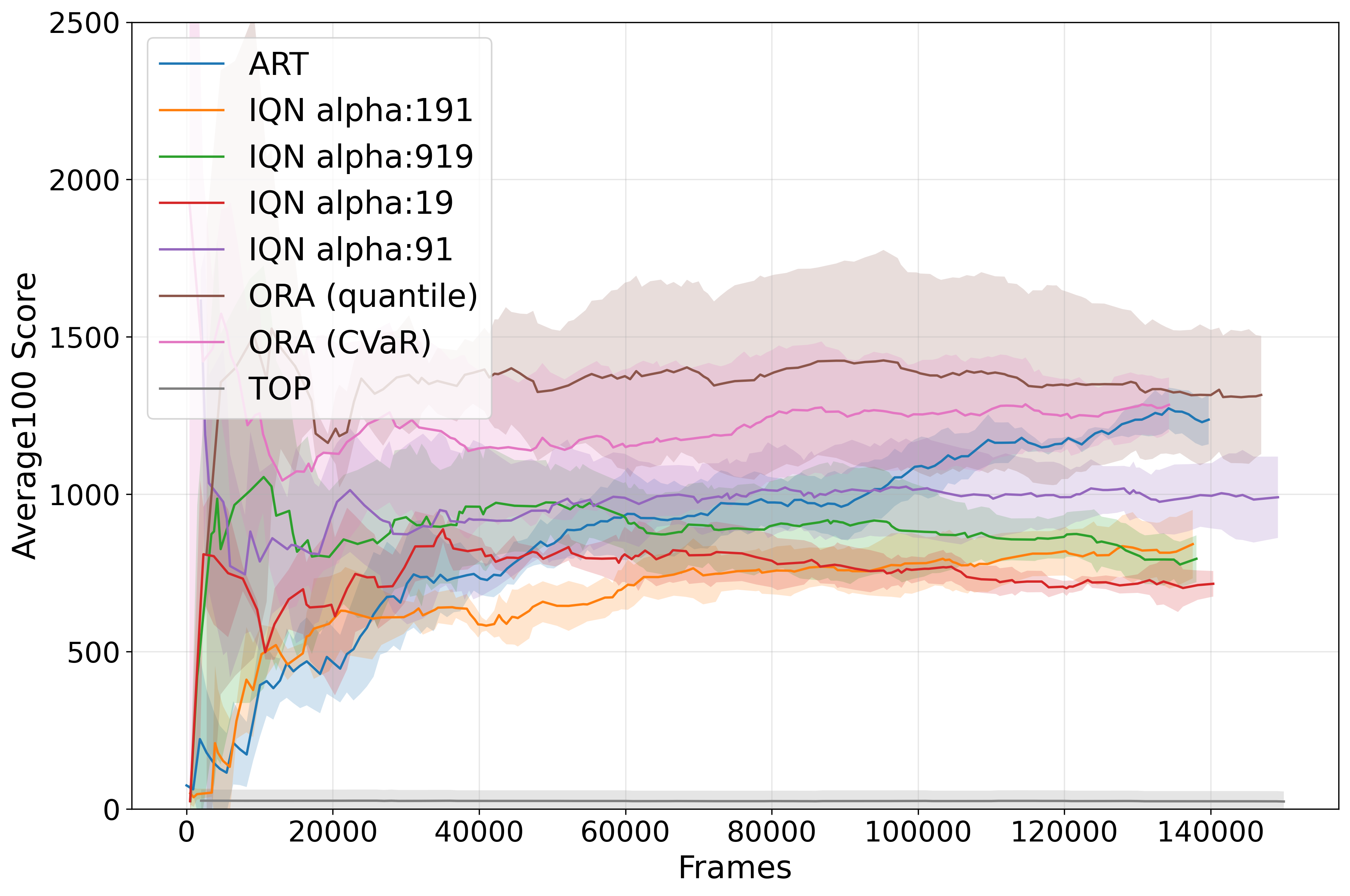}
        \caption{Hero}
        \label{fig:image_a}
    \end{subfigure}
    \hfill
    \begin{subfigure}[b]{0.5\textwidth}
        \includegraphics[width=\textwidth]{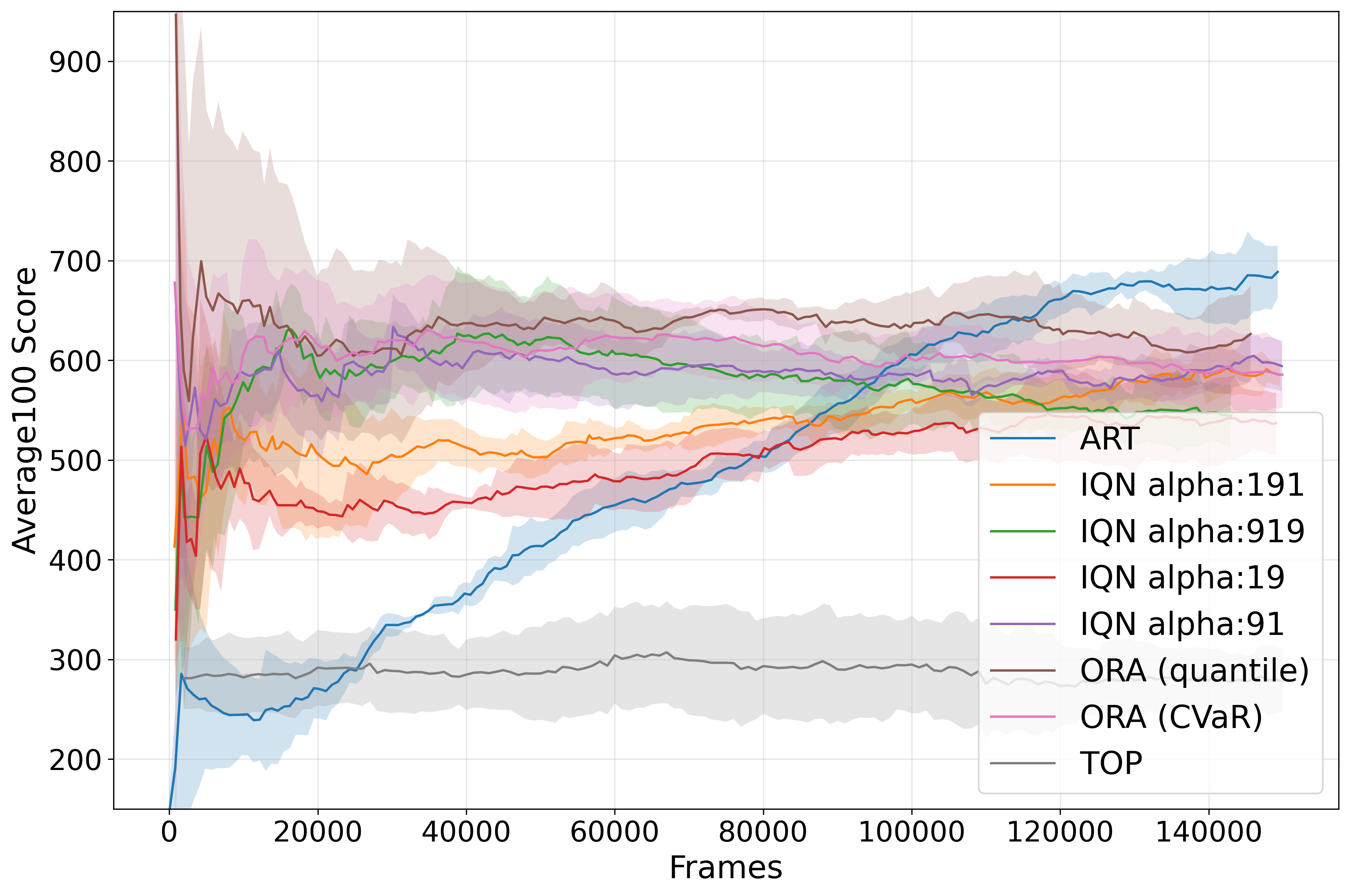}
        \caption{MsPacman}
        \label{fig:image_b}
    \end{subfigure}
    \hfill
    \begin{subfigure}[b]{0.5\textwidth}
        \includegraphics[width=\textwidth]{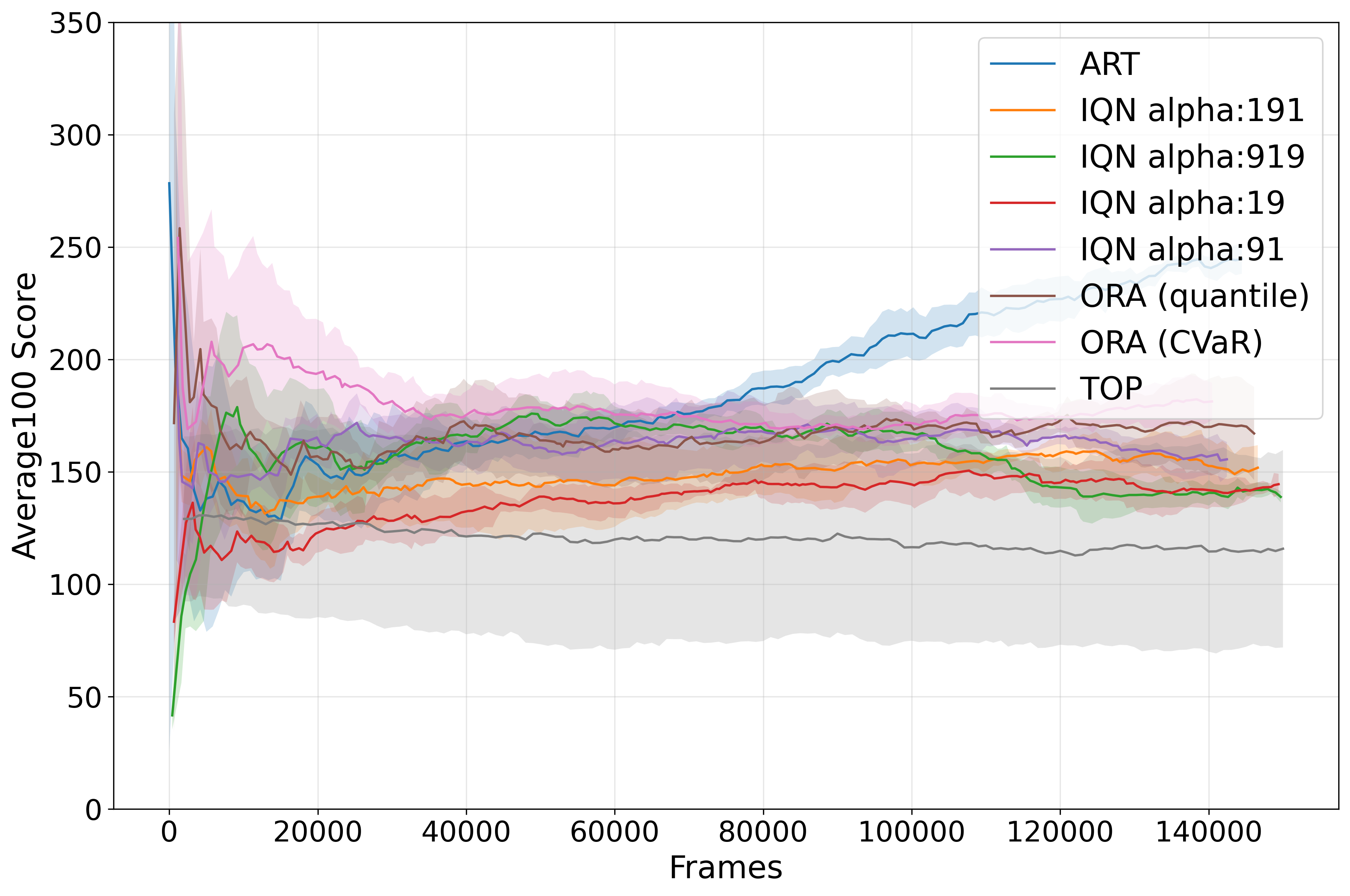}
        \caption{SpaceInvaders}
        \label{fig:image_c}
    \end{subfigure}
    \caption{The average score of last 100 episodes curves (mean $\pm$ std) of ORA, ART, TOP, and baseline algorithms across three Atari environments over 3 seeds.}
    \label{fig:three_images}
\end{figure}

{
\paragraph{Computational Test} 
Although the ensemble network is inseparable from our method, making it impossible to investigate the ablation improvement of the ensemble network, the test on ensemble size $K$ demonstrates that the ensemble network itself does not exhibit a significant correlation with performance improvement (Figure~\ref{fig:ablation k}).
Further computational tests reveal that the time overhead grows only slightly with the ensemble size $K$, while the associated memory overhead is negligible. This demonstrates that introducing the ensemble network does not incur a substantial computational burden. As a reference, we have also included the average time and space consumption of the TOP and IQN methods (the ensemble size is not applicable). The comparison further confirms that our method only requires slightly higher memory usage, and with the default $K$, the additional computational time is insignificant.
Note that aforementioned experiments are conducted on the same GPU-accelerated device.
Despite variations in model architectures of these approaches and consequently their hyperparameter configurations (which are specified in the following subsection), a uniform batch size of 64 was employed for all methods, with each using the parameter set determined to be optimal for it.
}

{
Table~\ref{Relation between alpha and epistemic uncertainty} records ORA’s left-truncated variance (LTV), risk parameters, and median epistemic uncertainty per training period for all three experiments. The experimental results of ORA (with distortion function CVaR and quantile) and baselines are shown in Figure~\ref{fig:three_images}.
Across three diverse Atari environments, our ORA method demonstrates statistically significant dominance ($p$ < 0.001), achieving massive to large Rank-Biserial effect sizes against the TOP baseline (0.996 in Hero, 0.875 in MsPacman, 0.566 in SpaceInvaders) and substantial effects against ART (0.672, 0.511, 0.388, respectively).
For these three environments, the discretization of $\mathcal{A}^{\prime}$ follows a grid size accuracy of 5, while in Nano Drone and knapsack, the grid size is bucketed to 1. The perturbation parameter $\eta$ for the FTPL algorithm is set to 0.5 in all experiments.
}
The hyper-parameter settings for CartPole, Atari games, Knapsack and Nano Drone Navigation tasks are listed in Table~\ref{Supp:Hyper-parameters C},~\ref{Supp:Hyper-parameters} and~\ref{Supp:Hyper-parameters N}. Parameters not listed are as default in IQN. The implementation and hyper-parameter settings of ART and TOP in Atari games and Knapsack environments follows the open source \citep{moskovitz2021tactical,liu2022adaptive}. The experiment was conducted on a Windows platform with an Intel i5-13400F and an Nvidia RTX 5060. The dependencies and libraries required for each task can be found in the code appendix.


\begin{table*}[!t]
  \caption{Relation between risk parameters and epistemic uncertainty}
  \label{Relation between alpha and epistemic uncertainty}
  \centering
  \begin{small}
  \begin{tabular}{lccc}
    \toprule
     ~ & CartPole & Nano Drone Navigation & Knapsack \\
    \midrule
    Risk parameter $\alpha$  & \raisebox{-.5\height}{\includegraphics[width=0.22\textwidth]{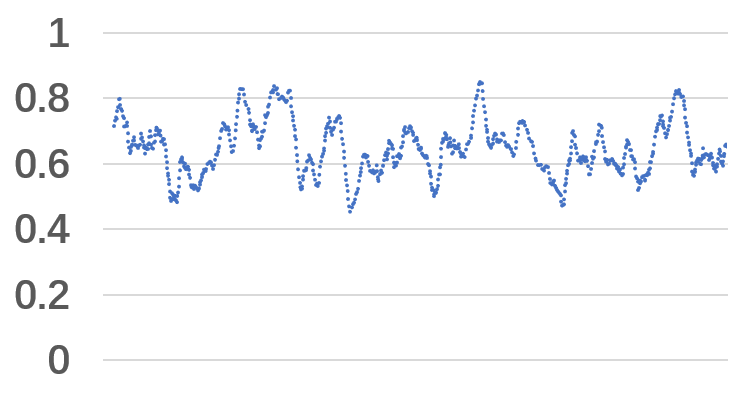}} & \raisebox{-.5\height}{\includegraphics[width=0.22\textwidth]{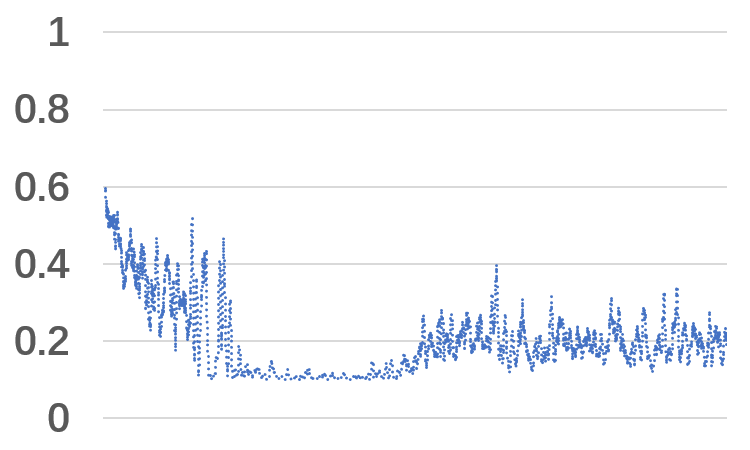}} & \raisebox{-.5\height}{\includegraphics[width=0.22\linewidth]{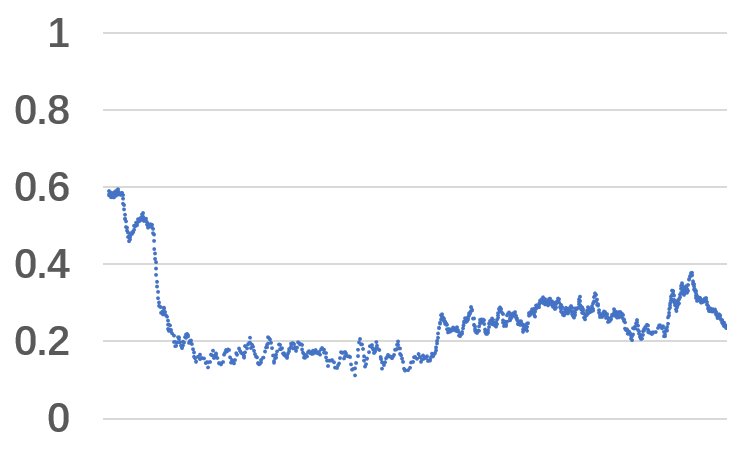}} \\
    \midrule
Epistemic uncertainty (LTV) & \raisebox{-.5\height}{\includegraphics[width=0.22\textwidth]{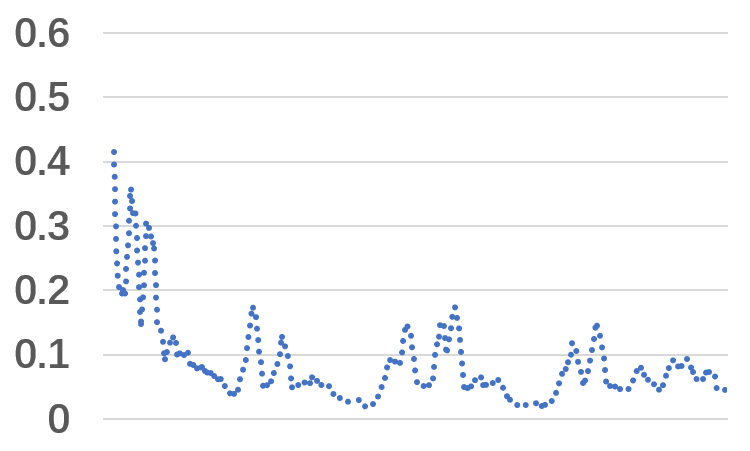}} & \raisebox{-.5\height}{\includegraphics[width=0.22\textwidth]{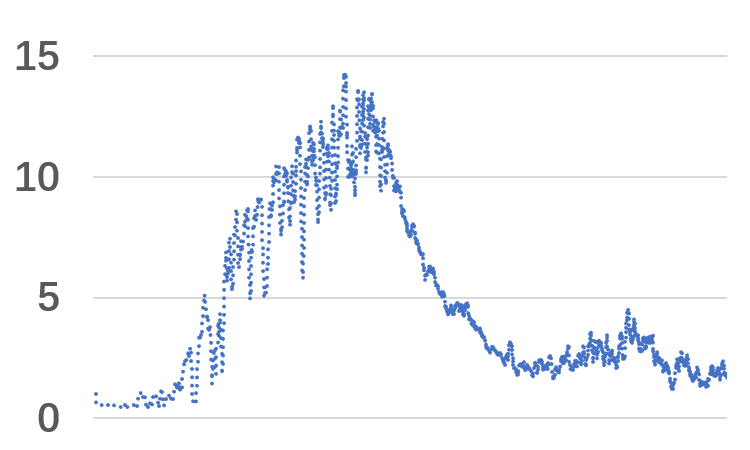}} & \raisebox{-.5\height}{\includegraphics[width=0.22\linewidth]{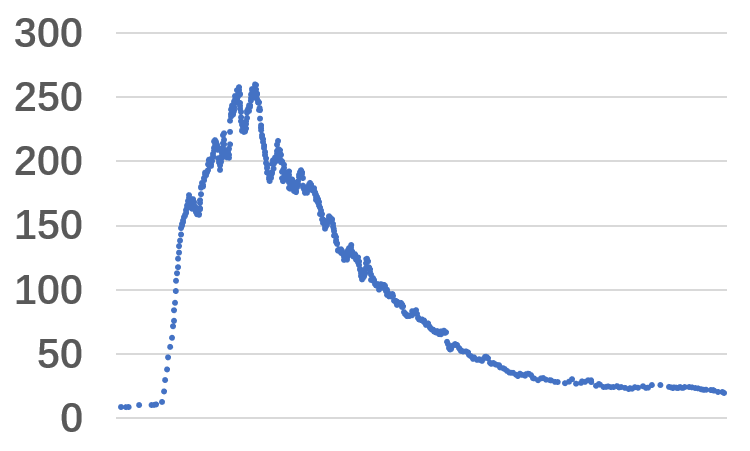}} \\
    \bottomrule
  \end{tabular}
  \end{small}
\end{table*}

  \begin{table}[h]
  \caption{Hyper-parameters in CartPole and Atari game environments}
  \label{Supp:Hyper-parameters C}
  \centering
  \begin{tabular}{ll}
    \toprule
    Parameter     & Value \\
    \midrule
    Optimizer & Adam \\
    Learning rate & 0.03 \\
    Batch size & 8 \\
    Replay buffer size & 100000 \\
    Target network update speed & 1 \\
    Discount factor & 0.99 \\
    Epsilon greedy & 0.1 (random action) \\
    Ensemble networks & 32\\
    Layer size & 256 \\
    Number of distortion samplings & 64\\
    Number of samplings for quantile loss approximation & 8\\
    \bottomrule
  \end{tabular}
\end{table}

  \begin{table}[H]
  \caption{Hyper-parameters in Knapsack}
  \label{Supp:Hyper-parameters}
  \centering
  \begin{tabular}{ll}
    \toprule
    Parameter     & Value \\
    \midrule
    Optimizer & Adam \\
    Learning rate & 1e-3 \\
    Batch size & 32 \\
    Replay buffer size & 4096 \\
    Target network update speed & 0.01 \\
    Discount factor & 1.0 \\
    Epsilon greedy & 0.1 (random action) \\
    Ensemble networks & 10 \\
    Layer size & 64 \\
    Number of episodes & 50000 \\
    \bottomrule
  \end{tabular}
\end{table}

   \begin{table}[H]
  \caption{Hyper-parameters in Nano Drone Navigation}
  \label{Supp:Hyper-parameters N}
  \centering
  \begin{tabular}{ll}
    \toprule
    Parameter     & Value \\
    \midrule
    Optimizer & Adam \\
    Learning rate & $2\times 10^{-4}$ \\
    Experience replay buffer with size & $6\times 10^4$ \\
    Discount factor & 0.99 \\
    State update period & 0.1 $[s]$ \\
    Maximum velocity & 1 $[m/s]$ \\
    Radius of drone & 0.05 $[m]$ \\
    Goal distance & 6 $[m]$\\
    Goal-reaching threshold & 0.5 $[m]$ \\
    Safety margin & 0.2 $[m]$ \\
    Number of episodes & 1000 \\
    Number of distortion samplings & 64\\
    Batch size & 32\\
    Maximum time-steps & 200 \\
    Number of samplings for quantile loss approximation & 16\\
    Number of velocity magnitudes & 3\\
    Lower bound on CVaR risk parameter & 0.1\\
    Learning rate in ART & 0.5 \\
    Number of hidden layers & 3 \\
    Number of units per layer & 512\\
    Ensemble networks &  10 \\
    \bottomrule
  \end{tabular}
\end{table}

\end{document}